\documentclass[lettersize,journal]{IEEEtran}
\usepackage{amsmath,amsfonts}
\usepackage{array}
\usepackage[caption=false,font=normalsize,labelfont=sf,textfont=sf]{subfig}
\usepackage{textcomp}
\usepackage{stfloats}
\usepackage{url}
\usepackage{verbatim}
\usepackage{multirow}
\usepackage{graphicx}
\usepackage{cite}
\usepackage{makecell}
\usepackage{booktabs}
\usepackage{bm}
\usepackage{amsthm}
\usepackage{amsmath}
\usepackage[lined,boxed,linesnumbered,ruled]{algorithm2e}
\newenvironment{ntp}{%
  
  \begin{abstract}%
}{%
  \end{abstract}%
}
\DeclareMathSymbol{\Tau}{\mathalpha}{operators}{84}
\newtheorem{theorem}{Theorem}
\newtheorem{definition}{Definition}

\begin{document}

\title{CTS-CBS: A New Approach for Multi-Agent Collaborative Task Sequencing and Path Finding}

\author{Junkai~Jiang$^{1}$, Ruochen~Li$^{1}$, Yibin~Yang$^{1}$, Yihe~Chen$^{1}$, Yuning~Wang$^{1}$, Shaobing~Xu$^{1\ast}$, and Jianqiang~Wang$^{1\ast}$
\thanks{This research was funded by National Natural Science Foundation of China, Science Fund for Creative Research Groups (Grant No. 52221005) and National Natural Science Foundation of China (Grant No. 52131201).}
\thanks{$^{1}$Junkai~Jiang, Ruochen~Li, Yibin~Yang, Yihe~Chen, Yuning~Wang, Shaobing~Xu and Jianqiang~Wang are with the School of Vehicle and Mobility, Tsinghua University, Beijing, China.}
\thanks{$^{\ast}$Corresponding author: Shaobing~Xu (shaobxu@tsinghua.edu.cn) and Jianqiang~Wang (wjqlws@tsinghua.edu.cn).}
}

\maketitle

\begin{abstract} 
This paper addresses a generalization problem of Multi-Agent Pathfinding (MAPF), called Collaborative Task Sequencing - Multi-Agent Pathfinding (CTS-MAPF), where agents must plan collision-free paths and visit a series of intermediate task locations in a specific order before reaching their final destinations. To address this problem, we propose a new approach, Collaborative Task Sequencing - Conflict-Based Search (CTS-CBS), which conducts a two-level search. In the high level, it generates a search forest, where each tree corresponds to a joint task sequence derived from the jTSP solution. In the low level, CTS-CBS performs constrained single-agent path planning to generate paths for each agent while adhering to high-level constraints. We also provide theoretical guarantees of its completeness and optimality (or sub-optimality with a bounded parameter). To evaluate the performance of CTS-CBS, we create two datasets, \texttt{CTS-MAPF} and \texttt{MG-MAPF}, and conduct comprehensive experiments. The results show that CTS-CBS adaptations for MG-MAPF outperform baseline algorithms in terms of success rate (up to 20 times larger) and runtime (up to 100 times faster), with less than a 10\% sacrifice in solution quality. Furthermore, CTS-CBS offers flexibility by allowing users to adjust the sub-optimality bound $\omega$ to balance between solution quality and efficiency. Finally, practical robot tests demonstrate the algorithm's applicability in real-world scenarios.
\end{abstract}

\begin{ntp}
This paper presents CTS-CBS, an optimization algorithm designed to solve the Collaborative Task Sequencing Multi-Agent Pathfinding (CTS-MAPF) problem. The proposed method significantly improves both success rate and efficiency compared to traditional algorithms. Practitioners working in autonomous robotics or warehouse automation could benefit from this approach to enhance task sequencing and path planning in systems requiring coordination among multiple agents. The algorithm can be integrated with existing multi-robot systems with minimal changes to the infrastructure. Future work includes exploring additional strategies to further enhance the efficiency of CTS-CBS, considering the kinematic constraints of agents to broaden its application scope, as well as extending the algorithm to dynamic and lifelong scenarios.
\end{ntp}



\section{Introduction} \label{Sec:Intro}
\IEEEPARstart{M}{ulti-agent} path finding (MAPF) seeks to generate conflict-free paths for multiple agents in a shared environment. Each agent needs to navigate from the start to its designated goal without colliding with others \cite{stern2019multi}. MAPF is widely applicable across various fields, including robotics \cite{ma2017lifelong}, warehouses \cite{honig2019persistent} and autonomous vehicles \cite{wen2022cl}. In this paper, we explore a generalization of MAPF, referred to as Collaborative Task Sequencing – Multi-Agent Path Finding (CTS-MAPF), where agents are required to visit a series of intermediate task locations before reaching their destinations. In CTS-MAPF, agents must adhere to collision-free constraints, while determining the optimal sequence for visiting these task locations in order to minimize the overall cost of completing the entire operation, such as the total time or total distance. Fig.~\ref{fig:ctscbs_example} presents a simple example of CTS-MAPF.

\begin{figure}[b]
    \centering
    \includegraphics[width=0.45\linewidth]{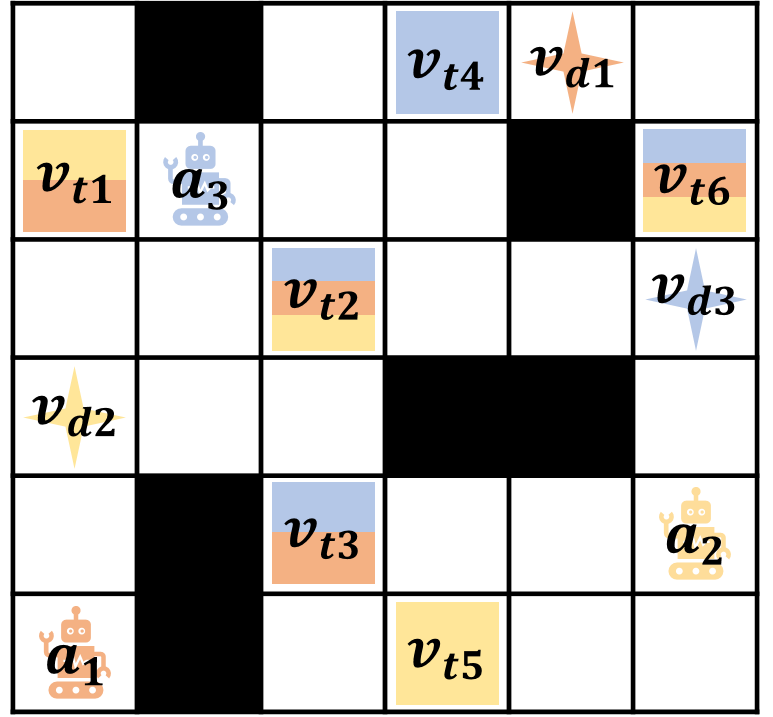} 
    \caption{A simple example of CTS-MAPF. There are three agents ($a_i, i=1,2,3$) at their initial locations; their respective goal locations are denoted as (($v_di, i=1,2,3$) and the task locations are denoted as $v_ti$. The colors marked on the task locations indicate the specific agents required to reach there. For example, $v_t1$ requires agents $a_1$ and $a_2$ to reach, while $v_t2$ requires all three agents to reach it. For task locations that require multiple agents, there is no specified order in which the agents must obey. The agents need to plan collision-free paths to visit their respective task locations and ultimately return to their destinations.}
    \label{fig:ctscbs_example}
    \vspace{-3mm}
\end{figure}

Compared to MAPF, CTS-MAPF also has a wide range of applications but is relatively more specific. For instance, in a hazardous monitoring scenario, different physical information (e.g., temperature, humidity, electromagnetic intensity, concentrations of harmful substances) need to be measured at several fixed monitoring locations. Each mobile robot is equipped with different sensors, so they must arrive at the monitoring locations that require their specific sensors to collect information, and eventually reach their respective destinations. Another application example is an unmanned market, where each shelf holds a specific category of goods, and each mobile robot is responsible for picking items for a single order, which may require a combination of various goods. Each mobile robot needs to plan its path to visit different shelves to collect items, and finally reach the goal location, all while ensuring that all paths are collision-free. Fig.~\ref{fig:applications} (a) and (b) illustrate the scenarios of the aforementioned two applications respectively.

Therefore, the essence of CTS-MAPF lies in planning both the task sequence and the collision-free path for each robot, making it an \emph{integrated task sequencing and path planning} problem. Since the task sequence of one robot may be influenced by the planning results of others, the task sequencing process is \emph{collaborative}.

Solving CTS-MAPF is challenging, especially when completeness or optimality guarantees are required, as it necessitates addressing the problem in both MAPF and task sequencing. If the order of tasks is predefined, meaning each robot visits a series of target locations (including the goal) in a specified sequence, then CTS-MAPF reduces to MAPF or its straightforward extension. However, MAPF itself is NP-hard. If collisions between robots are ignored, CTS-MAPF reduces to a combination of several single-agent TSPs, called joint TSP (jTSP) in this paper. However, the single-agent TSP is also NP-hard. Therefore, solving CTS-MAPF is extremely difficult, as it requires tackling the challenges of both MAPF and TSP, and effectively integrating their solving techniques.

This work aims to formalize and study CTS-MAPF, where robots need to achieve collaborative task sequencing while planning collision-free paths, and minimizing the \emph{flowtime} (the sum of the finish times of all agents at the last goal locations of their assigned tasks). We seek to develop a new approach capable of finding the optimal solution for CTS-MAPF or a bounded suboptimal one to balance optimality and efficiency.

\begin{figure}[htbp]
  \centering
  \begin{minipage}{0.48\linewidth}
    \centering
    \includegraphics[width=0.9\linewidth]{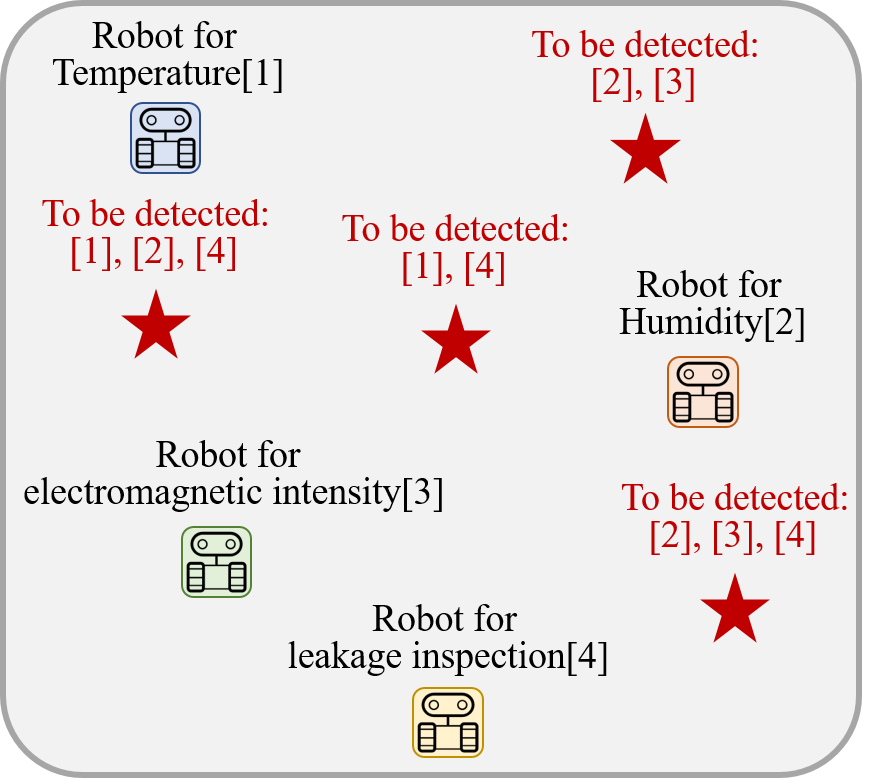} 
    \centerline{(a)}
  \end{minipage}
  \begin{minipage}{0.48\linewidth}
    \centering
    \includegraphics[width=0.9\linewidth]{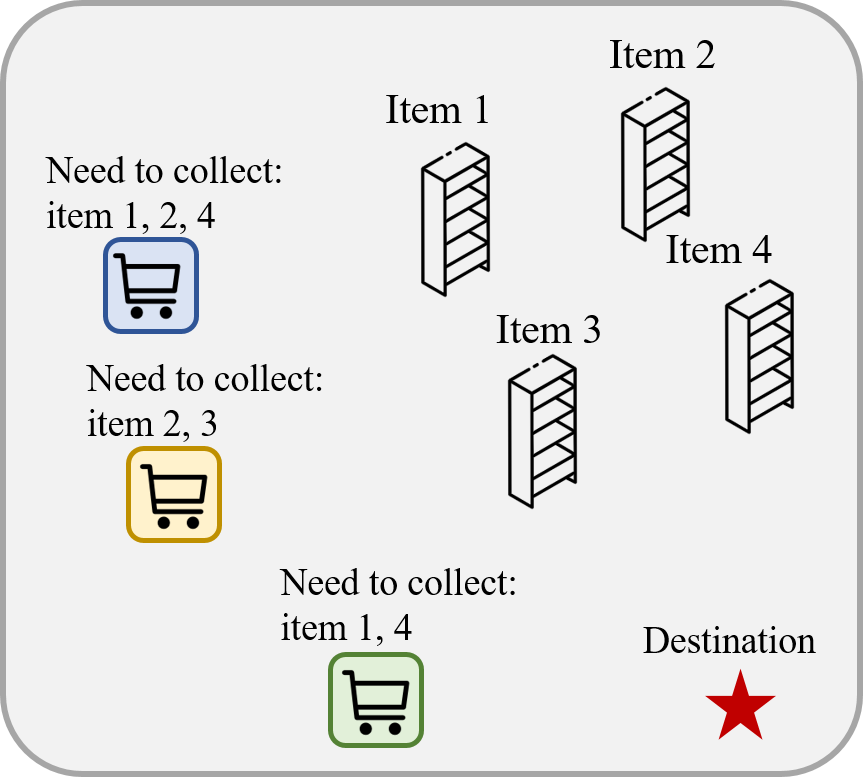}
    \centerline{(b)}
  \end{minipage}
  \caption{The possible applicable scenarios of CTS-MAPF. (a) for hazardous monitoring and (b) for unmanned market.}
  \label{fig:applications}
  \vspace{-3mm}
\end{figure}

\subsection{Related Work}
\subsubsection{MAPF}
MAPF seeks to find collision-free paths for all agents, and minimize the total travel times. To solve the MAPF problem efficiently and effectively, researchers have made significant efforts. MAPF algorithms can generally be divided into two categories: suboptimal solvers and optimal solvers. Typical methods for the former include PUSH AND SWAP \cite{luna2011push}, Hierarchical Cooperative A* \cite{silver2005cooperative} and Prioritized Planning \cite{vcap2015prioritized}, which are usually easy to deploy but are non-complete and non-optimal. On the other hand, optimal solvers are capable of solving MAPF optimally, with methods including Independence Detection \cite{standley2010finding}, M* \cite{wagner2011m}, Constraints Programming \cite{bartak2017modeling}, and Conflict-Based Search (CBS) \cite{sharon2015conflict}. Among these, CBS has attracted a lot of attention in recent years due to its effective performance. It executes a bi-level search, where the high level runs A* on the constraint tree to resolve collisions and the low level runs A* to plan paths for single agents. CBS has been widely studied and extended (e.g., ECBS \cite{barer2014suboptimal}, ICBS \cite{boyarski2015icbs}, safer-CBS \cite{liu2024safer}). One significant direction of extension is exploring bounded suboptimal solvers for MAPF to balance the trade-off between solving efficiency and optimality. Since the MAPF problem definition does not include intermediate task locations, these methods cannot be directly applied to solving the CTS-MAPF problem.

\subsubsection{TSP}
In CTS-MAPF, agents need to determine the sequence of task locations, which can be modeled as a TSP. The TSP can be formulated as an integer linear program \cite{papadimitriou2013combinatorial}. 

TSP is a classic NP-hard problem. Its solving methods can be broadly categorized into three types: exact algorithms (such as dynamic programming \cite{bellman1962dynamic} and branch-and-bound methods \cite{padberg1991branch}), approximation algorithms (such as greedy algorithms \cite{johnson1997traveling} and minimum spanning tree-based algorithms \cite{lawler1985traveling}), and heuristic algorithms (such as genetic algorithms \cite{ray2007genetic} and ant colony optimization \cite{dorigo1997ant}). Additionally, there are some well-established solvers, such as GUROBI and LKH, that can solve TSP problems by appropriate methods based on user requirements. In CTS-MAPF, it is necessary to solve multiple instances of TSP to obtain the optimal task sequencing results. Therefore, TSP can be considered as the upper-level problem of CTS-MAPF, which will be detailed in Section \ref{Sec:CTSCBS}.

\subsubsection{MA-ITSPP}
The multi-agent integrated task sequencing and path planning (MA-ITSPP) problem has received extensive investigation recently. Most studies consider it an extension of MAPF and explore solving methods based on CBS. For instance, CBS-TA (task assignment) was proposed in \cite{honig2018conflict}, where the authors investigated the integration of TA with CBS to yield optimal solutions. Despite the different configurations, the concepts of $K$-best solutions and search forest presented in the paper are particularly noteworthy. The definition of MG (multi goal)-MAPF is similar to CTS-MAPF, except that agents do not have specified destinations. \cite{surynek2021multi} introduced the Hamiltonian CBS algorithm to address the MG-MAPF problem. Multi-goal task assignment and path finding (MG-TAPF) is another related issue \cite{zhong2022optimal}, but it specifies the sequence of tasks for each group in the problem definition. \cite{ren2023cbss} explored the Multi-Agent Combinatorial Path Finding problem, where each task can be completed by any of the multiple agents. To sum up, there is currently no direct research addressing the CTS-MAPF problem, and even similar problems (such as MG-MAPF) lack solving methods that balance quality and efficiency.

\subsection{Motivation and Contribution}
Existing methods explore limited aspects of CTS-MAPF problems. In this paper, we propose a new approach, Collaborative Task Sequencing – Conflict-Based Search (CTS-CBS), tailored for the general case of CTS-MAPF. Specifically, by integrating algorithms of TSP and MAPF, CTS-CBS follows a best-first manner to ensure the optimality of the solution. Similar to other CBS-based algorithms, CTS-CBS conducts a two-level search. In the high level, a search forest is generated, with each search tree corresponding to a joint task sequence determined by the solution of the jTSP. One contribution of this paper is providing the $K$-best solutions for the jTSP, as the root node of the $K$-th search tree in the search forest must match the $K$-th best solution (detailed in Section \ref{Sec:CTSCBS}). In the low level, CTS-CBS runs constrained single-agent path planning to generate paths for each agent while satisfying the constraints determined in the high level.

CTS-CBS is proved to be both complete and optimal (Section \ref{Sec:TheoAnal}). Additionally, CTS-CBS can also be bounded suboptimal if a boundary parameter is given. With a small boundary parameter, CTS-CBS can compute a near-optimal solution (and it is exactly optimal when the parameter is 0), though the computational burden is higher. Conversely, with a larger boundary parameter, CTS-CBS tends to produce a feasible solution more efficiently. Furthermore, If the boundary parameter is infinite, CTS-CBS degenerates into a sequential method, which first computes the optimal joint task sequence (without considering inter-agent conflicts) and then calculates collision-free paths for this sequence. Under this condition, the algorithm loses both completeness and optimality guarantees.

To sum up, the main contributions of this paper include:
\begin{enumerate}

\item CTS-CBS, a new approach to solve CTS-MAPF problems with optimality (or bounded suboptimality) and completeness guarantees.

\item A new method to compute $K$-best solutions for the jTSP.

\item Three adaptations of CTS-CBS for MG-MAPF problems which outperform all baseline algorithms.

\item Practical robot tests that verify the applicability of the algorithm in real-world robotic scenarios.

\end{enumerate}

The rest of this paper is organized as follows. In Section \ref{Sec:prelim}, we describe the CTS-MAPF problem and review the CBS algorithm. We then introduce the CTS-CBS approach in Section \ref{Sec:CTSCBS} and analyze it theoretically in Section \ref{Sec:TheoAnal}. The experiments are presented in Section \ref{Sec:Experiments} and finally, Section \ref{Sec:Conc} concludes this paper and provides directions for future work.

\section{Preliminaries} \label{Sec:prelim}
\subsection{Problem Definition}

The CTS-MAPF problem can be defined by a nine element tuple $\langle G=(V,E), A, \mathcal{O}, V_s, V_d, V_t, f, \mathcal{\Tau}, \Pi \rangle$. $A = \{a^1, a^2, \ldots, a^N\}$ is a set of $N$ agents, moving in a workspace denoted as an undirected graph $G=(V,E)$, where $V$ is the set of vertices (locations) and $E \subseteq V \times V$ is the set of edges connecting vertices. For any two vertices $v_1, v_2 \in V$, the cost (representing the minimal traveling time) of the edge connecting them can be denoted as $\text{cost}(v_1, v_2)$, which is a positive number. In this paper, we use $[N]$ to denote the set $\{1,2,\ldots,N\}$ and superscript $i$ to represent the variables related to agent $a^i$. Some static obstacles are randomly distributed in the environment, occupying the workspace $\mathcal{O}$.

$V_s$ and $V_d$ are the set of starts and destinations, and $v_s^i$ and $v_d^i$ represent the corresponding variables of agent $a^i$. $V_t$ is the set of task positions (for simplicity, we will use the term \textit{tasks} in the following content). There are $M$ tasks in total, denoted as $v_{t1}, v_{t2}, \ldots, v_{tM}$, and each task may require some agents to visit. Let $f$ denote the function from $V_t$ to $A$. That is, for any $v_t \in V_t$, $f(v_t)$ is a subset of $A$, representing the set of agents that need to visit $v_t$.

As mentioned earlier, in the CTS-MAPF problem, it is necessary to simultaneously solve the task sequencing and path planning problems. Let $\tau^i$ represent the target sequence of agent $a^i$ (including the start $v_s^i$, destination $v_d^i$ and all the $v_t$ where $a^i \in f(v_t)$). Specifically, $\tau^i = \{v_s^i, u_1^i, \ldots, u_k^i, v_d^i\}$, where $u_j^i$ represents the $j$-th task visited by agent $a^i$. A path $\pi^i$ is considered to follow $\tau^i$ if agent $a^i$ can visit all the tasks in the same order as in $\tau^i$ along $\pi^i$. We use $\mathcal{\Tau} = \{\tau^i \mid i \in [N]\}$ and $\Pi = \{\pi^i \mid i \in [N]\}$ to represent the joint task sequence and joint path for all the agents, respectively.

Now we introduce another important concept in the CTS-MAPF problem: cost calculation. The overall objective is to minimize the \emph{flowtime}, which means minimizing the cost of the joint path $\Pi$ while satisfying the constraints. The cost of $\Pi$ is defined as $\text{cost}(\Pi) = \sum \text{cost}(\pi^i)$. It is important to note that $\pi^i$ may include wait actions to avoid conflicts, and the cost resulting from these actions must also be included in $\text{cost}(\pi^i)$. Correspondingly, $\mathcal{\Tau}$ also has a cost defined as $\text{cost}(\mathcal{\Tau}) = \sum \text{cost}(\tau^i)$. However, the calculation of $\text{cost}(\tau^i)$ ignores all conflicts and represents the sum of the minimum costs between all adjacent vertices in $\tau^i$, i.e., $\text{cost}(\tau^i) = \text{cost}(v_s^i, u_1^i) + \sum \text{cost}(u_j^i, u_{j+1}^i) + \text{cost}(u_k^i, v_d^i)$. The concept of $\text{cost}(\mathcal{\Tau})$ will be used in the CTS-CBS algorithm.

If the CTS-MAPF is viewed as an optimization problem, several constraints need to be satisfied as follows:
\begin{enumerate}
    \item \textbf{Boundary Constraints}: For each agent $a^i$, it needs to start from $v_s^i$ and end at $v_d^i$.
    \item \textbf{Task Completion Constraints}: For each task $v_t$, all agents in $f(v_t)$ need to visit $v_t$ at least once.
    \item \textbf{Agent Behavior Constraints}: Each agent may have up to 5 possible actions: waiting or moving one step in one of the four cardinal directions.
    \item \textbf{Static Collision Constraints}: All agents must not collide with static obstacles in the environment.
    \item \textbf{Inter-Agent Collision Constraints}: Agents must not collide with each other. The inter-agent collisions (conflicts) includes two types: vertex conflicts and edge conflicts, as defined in MAPF.
\end{enumerate}

Considering the summarized elements, the CTS-MAPF problem can be formulated as

\begin{equation}\label{eq:ocp}
\begin{aligned}
    \min & \quad \quad \text{cost}(\Pi)\\
    \textrm{s.t.}& \quad  \textrm{Boundary Constraints},\\
  & \quad \textrm{Task Completion Constraints},\\
  & \quad \textrm{Agent Behavior Constraints},\\
  & \quad \textrm{Static Collision Constraints}, \\
  & \quad \textrm{Inter-Agent Collision Constraints}. 
\end{aligned}
\end{equation}

\subsection{Review of CBS}
CBS, which is an optimal and complete algorithm for MAPF problem, operates in a bi-level manner. The high-level search employs a binary search tree structure, called constraint tree ($CT$), whose nodes contain a tuple of constraints, solution, and cost, denoted as $(\Omega, \Pi, g)$. $\Omega$ is a set of constraints, each element prevents a specific agent from occupying a particular vertex or traversing a particular edge at a given time. $\Pi$ is the joint path that connects the starts and destinations for all the agents, and $g$ is the cost of $\Pi$, i.e., $g = g(\Pi) = \text{cost}(\Pi)$.

For each node within $CT$, a low-level search is executed, whose purpose is to determine the shortest path for each agent while adhering to the constraints in $\Omega$. Typically, the A*-like algorithm is employed for this purpose. Once paths for all agents are computed in a node, CBS checks for conflicts. Upon detecting a conflict, CBS resolves it by creating two new child nodes in $CT$. Each child node introduces an additional constraint to $\Omega$ aimed at resolving the conflict by restricting one of the conflicting agents. This process of conflict resolution continues until a set of conflict-free paths (i.e., a solution) is found.

CBS guarantees optimality through a best-first mechanism. Nodes in $CT$ are managed using a priority queue based on the cost $g$. CBS starts constructing $CT$ from the root node $P_{\text{root}}(\emptyset, \Pi_0, g(\Pi_0))$, where $\Pi_0$ is the joint path obtained by running the low-level search independently for all agents without considering inter-agent conflicts. In each iteration, the node with the lowest $g$-value, denoted as $P_f(\Omega_f, \Pi_f, g(\Pi_f))$ (the subscript $f$ stands for 'father'), is extracted from the priority queue for expansion. Then, every pair of paths in $\Pi_f$ is examined for conflicts. If there are no conflicts, then $\Pi_f$ is the solution, which is guaranteed to be optimal. Otherwise, if a conflict $(a^i, a^j, v, t)$ or $(a^i, a^j, e, t)$ is detected, it is split into two constraints: $(a^i, v, t)$ and $(a^j, v, t)$ or $(a^i, e, t)$ and $(a^j, e, t)$. Then two new nodes $P_{c1}(\Omega_{c1}, \Pi_{c1}, g(\Pi_{c1}))$ and $P_{c2}(\Omega_{c2}, \Pi_{c2}, g(\Pi_{c2}))$ (the subscript $c$ stands for 'child') are generated, where $\Omega_{c1} = \Omega_f \cup \{(a^i, v, t) \lor (a^i, e, t)\}$ and $\Omega_{c2} = \Omega_{f} \cup \{(a^j, v, t) \lor (a^j, e, t)\}$. $\Pi_{c1}$ and $\Pi_{c2}$ are the joint paths obtained by rerunning the low-level search under the new constraint sets, and $g(\Pi_{c1})$ and $g(\Pi_{c2})$ are their corresponding costs. The new child nodes are then inserted back into the priority queue based on their $g$-value for the next iteration.

\section{Collaborative Task Sequencing - Conflict-Based Search} \label{Sec:CTSCBS}

In this section, we first describe the overall process of the CTS-CBS algorithm. Then, we explain two key components that ensure the algorithm’s bounded suboptimality: $K$-best joint task sequencing and necessity checking for new root nodes. Finally, we discuss some implementation details to improve the algorithm’s efficiency.

\subsection{CTS-CBS Algorithm}

CTS-CBS also follows a bi-level search architecture. We visualize the main process of CTS-CBS in Fig.~\ref{fig:ctscbs}. Unlike CBS, in the high-level search, CTS-CBS generates a search forest, which contains multiple search constraint trees ($CT$s). Each $CT$ corresponds to a joint task sequence $\mathcal{\Tau}$, meaning that every node in the same $CT$ has joint paths $\Pi$ that follow the corresponding $\mathcal{\Tau}$. As previously mentioned, the joint task sequence $\mathcal{\Tau}_j^*$ corresponds to the $K$-th best solution of the jTSP problem, so the costs of the joint task sequences $\{\mathcal{\Tau}_j^*, j = 1, 2, \ldots\}$ are non-decreasing, i.e., $\text{cost}(\mathcal{\Tau}_1^*) \leq \text{cost}(\mathcal{\Tau}_2^*) \leq \text{cost}(\mathcal{\Tau}_3^*) \leq \ldots$. Within a single $CT$, the purpose of node expansion is to eliminate inter-agent collisions. Node expansion between $CT$s, however, needs to follow certain rules introduced later. In the low-level search of a node $P(\Omega, \Pi, g)$ in $CT_j$, an A*-like algorithm is employed to find the optimal joint path $\Pi$ for agents, that follow $\mathcal{\Tau}_j^*$ while satisfying the constraints $\Omega$.

\begin{figure}[htpb]
    \centering
    \includegraphics[width=0.95\linewidth]{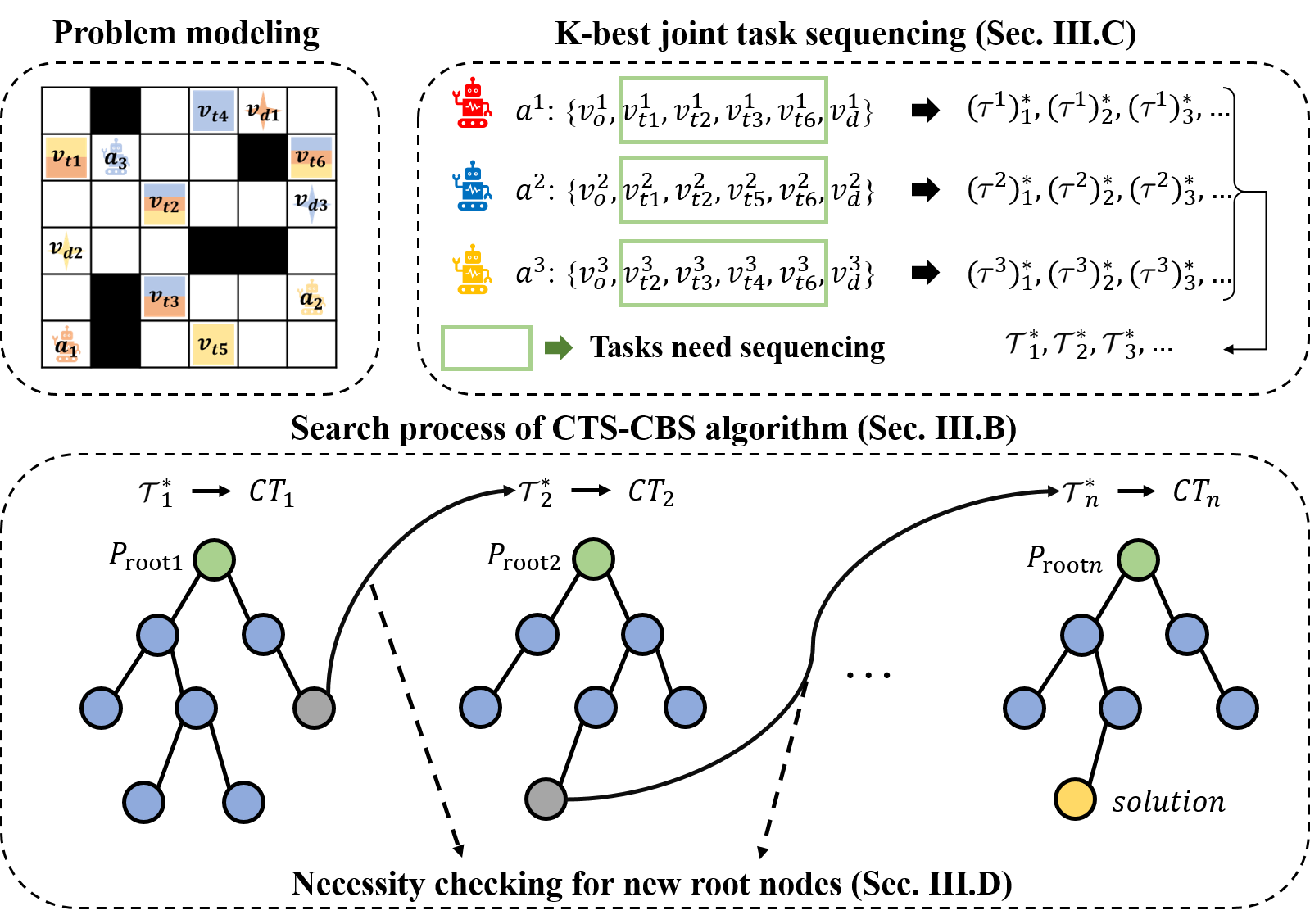} 
    \vspace{-3mm}
    \caption{The main process of CTS-CBS. The joint path $\Pi$ of every node $P$ in search tree $CT_j$ follows the same joint task sequence $\mathcal{\Tau}_j$. The node expansion within a single $CT$ is as CBS to resolve conflicts. If there is a need for a new search tree, the new root node is generated to construct it based on the next best joint task sequence obtained by $K$-best joint task sequencing. CTS-CBS operates in a best-first manner, ensuring that the first node containing a conflict-free path $\Pi$ found has the minimum cost $g$.}
    \label{fig:ctscbs}
\end{figure}

CTS-CBS starts constructing the first search tree $CT_1$ with the (first) best task sequence $\mathcal{\Tau}_1^*$. The corresponding root node is $P_{\text{root1}}(\emptyset, \Pi_0, g(\Pi_0))$. If there are no conflicts in $\Pi_0$, then $\Pi_0$ is the solution, which is of course optimal. If there are conflicts in $\Pi_0$, then two new nodes will be created to resolve the conflicts and added to the priority queue OPEN. Subsequently, a node $P(\Omega, \Pi, g)$ with the lower $g$-value is popped from the OPEN. In CBS, $P$ would directly undergo conflict detection and subsequent node expansion. However, in CTS-CBS, $P$ first needs to be compared with $\mathcal{\Tau}_2^*$ in terms of cost. If $g$ is greater than $\text{cost}(\mathcal{\Tau}_2^*)$, then the second search tree $CT_2$ with root node $P_{\text{root2}}$ will be generated, and conflict detection and node expansion will be performed on $P_{\text{root2}}$, while $P$ is re-added to the OPEN. Conversely, if $g$ is less than the cost of $\mathcal{\Tau}_2^*$, conflict detection and node expansion are performed on $P$, and $P_{\text{root2}}$ is added to the OPEN. Through this process, CTS-CBS operates in a best-first manner, ensuring that the first node with a conflict-free $\Pi$ found is optimal.

As illustrated in Fig.~\ref{fig:ctscbs}, there are two key components in CTS-CBS: $K$-best joint task sequencing (Section \ref{sub:Kbest}) and necessity checking for new root nodes (Section \ref{sub:nece-check}). Obtaining the $K$-best solutions for the jTSP problem is not easy; in this work, we propose a two-layer A*-like algorithm to solve this problem. The low level is designed for single-agent $K$-best task sequencing, while the upper level is designed for $K$-best joint task sequencing. The necessity checking for new root nodes is required, if a bounded suboptimal solution instead of the optimal one is desired to improve efficiency.

Now, we describe the detailed process of the CTS-CBS algorithm. The pseudo-code is presented in Algorithm~\ref{algo:ctscbs}. We first determine the best task sequence $\mathcal{\Tau}_1^*$ (line 1), create the first search tree with root node $P_{\text{root1}}$ (lines 2-4) whose constraint set is empty and joint path $\Pi$ is obtained by the $LowLevelSearch$. Then, $P_{\text{root1}}$ is inserted into OPEN (line 5). In each iteration of CTS-CBS (lines 6-23), the best node with the lowest $g$-value, denoted as $P_p$ (the subscript 'p' stands for 'popped'), is popped from the OPEN (line 7). Then, the process of necessity checking for new root nodes (Algorithm~\ref{algo:nece-check}) is performed to determine whether a new root node needs to be generated. This process returns node $P_n$ (the subscript 'n' stands for 'new') for subsequent conflict detection and node expansion (line 8). If there are no conflicts in ${P_n}.\Pi$, the algorithm returns ${P_n}.\Pi$, declares the solution found and terminates (lines 10-12). Otherwise, if there are conflicts, the first one is denoted as $C$ (line 13). For each agent involved in $C$, a new node ${P_n}'$ is generated (line 15), its $\Omega$, $\Pi$, and $g$ are updated (lines 16-18), and it is added to OPEN if feasible (lines 19-21). If OPEN is exhausted without finding a conflict-free joint path, the problem is unsolvable, and the algorithm returns false (line 24).

\begin{algorithm}[tb]
    \caption{The CTS-CBS algorithm}\label{algo:ctscbs}
        \SetKwInOut{Input}{Input}

        \Input{CTS-MAPF instance}
        $\mathcal{\Tau}_1^*$ $\gets$ $KBestJointSequencing(K=1)$\\
        $P_{\text{root1}}.\Omega \gets \emptyset$\\
        $P_{\text{root1}}.\Pi \gets LowLevelSearch(\mathcal{\Tau}_1^*, P_{\text{root1}}.\Omega)$\\
        $P_{\text{root1}}.g \gets \text{cost}(P_{\text{root1}}.\Pi)$\\
        Insert $P_{\text{root1}}$ to OPEN\\
        \While{OPEN not empty}{
            $P_{p}(\Omega, \Pi, g) \gets$ pop the best node from OPEN\\
            $P_{n}(\Omega, \Pi, g) \gets CheckNewRoot(P_{p}, \text{OPEN})$\\
            Validate $P_{n}.\Pi$ until a conflict occurs\\
            \If{$P_{n}.\Pi$ has no conflict}{
                \Return $P_{n}.\Pi$ \tcp{Solution found}
            }
            $C \gets$ first conflict $(a^i, a^j, v/e, t)$\\
            \ForEach{agent $a^i$ in $C$}{
                ${P_n}^{'} \gets$ new node\\
                ${P_n}^{'}.\Omega \gets P_n.\Omega \cup (a^i, v/e, t)$\\
                ${P_n}^{'}.\Pi \gets LowLevelSearch(\mathcal{\Tau}(P_n), P_{n}.\Omega)$\\
                \tcp{Only replan agent $a^i$'s path}
                ${P_n}^{'}.g \gets \text{cost}(P_n^{'}.\Pi)$\\
                \If{${P_n}^{'}.g < \infty$ }{
                    Insert ${P_n}^{'}$ to OPEN\\
                }
            }
        }
    \Return false 
\end{algorithm}

\vspace{-3mm}
\subsection{$K$-best Joint Task Sequencing} \label{sub:Kbest}
To compute the $K$-best joint task sequence, we employ a two-layer A*-like algorithm. The lower layer solves the single agent $K$-best task sequence problem, primarily utilizing the definition of restricted TSP and the partition method \cite{ren2023cbss}. In the upper layer, we combine all the single agent $K$-best task sequences and ultimately obtain the $K$-best joint task sequence, thus deriving the $K$-best solutions for jTSP.

\subsubsection{Single agent $K$-best task sequencing}
Restricted TSP (rTSP) can be defined as follows: Given a graph $G = (V, E)$, and two subsets of edge set $E$, denoted as $I_e$ and $E_e$, the rTSP problem seeks to find an optimal solution $\tau^*$ such that $\tau^*$ includes all edges in $I_e$ and excludes all edges in $E_e$. rTSP needs to be used in conjunction with the partition method, which is to use a priority queue (denoted as OPENs, where 's' stands for 'single-agent') to store all possible $K$-best solutions. In each iteration, the current optimal solution is popped from the queue and partitioned. The solutions to the rTSP problem under each partition are then added back into OPENs.

Algorithm~\ref{algo:singleagentkbest} shows the pseudo-code. Given a graph $G_{SA}$, we first generate the cost matrix $M_C$ (line 1), create the set $N_{L,L}$ (to store $K$-best solutions), and initialize $I_{e,1}$ and $E_{e,1}$ (for included and excluded edges of the best solution) as empty sets. Then, we solve the rTSP problem with $M_C$, $I_{e,1}$, and $E_{e,1}$, derive the best solution $\tau_1^*$, and calculate its corresponding cost as $c_1$ (lines 2-5). The four-element node $(c_1, \tau_1^*, I_{e,1}, E_{e,1})$ is inserted into OPENs (line 6). In the $k$-th iteration of the algorithm (lines 7-24), the current best node is popped from OPENs, and $\tau_k^*$ is added to $N_{L,L}$ (lines 8-9). If solutions is sufficient, the algorithm terminates and returns $N_{L,L}$ (lines 10-12). Otherwise, denote all $m$ edges in $\tau_k^*$ as $\mathcal{L} = \{ e_1, \ldots, e_m \}$ (line 13). For each $l$ edge, generate a new node, update its included edges set and excluded edges set to $I_{e,k+1,l}$ and $E_{e,k+1,l}$, respectively, solve the rTSP problem again to obtain $\tau_{k+1,l}^*$, and calculate its cost $c_{k+1,l}$ (lines 15-19). If $\tau_{k+1,l}^*$ is feasible, add the node to OPENs (lines 20-22).

This method's optimality is proven in \cite{ren2023cbss}. Algorithm~\ref{algo:singleagentkbest} can output single agent $K$-best task sequences, if they exist.

\begin{algorithm}[hbtp]
    \caption{Single agent $K$-best task sequencing}\label{algo:singleagentkbest}
        \SetKwInOut{Input}{Input}

        \Input{$K$, single agent graph $G_{SA}$}
        $M_C \gets GenerateCostMatrix(G_{SA})$\\
        $N_{L, L} \gets \emptyset$\\
        $I_{e, 1} \gets \emptyset$, $E_{e, 1} \gets \emptyset$\\
        $\tau_1^* \gets SolveRTSP(M_C, I_{e, 1}, E_{e, 1})$\\
        $c_1 \gets \text{cost}(\tau_1^*)$\\
        Insert $\left( c_1, \tau_1^*, I_{e, 1}, E_{e, 1} \right)$ to OPENs \\
        \While{\textnormal{OPENs} not empty}{
            $(c_k, \tau_k^*, I_{e, k}, E_{e, k}) \gets$ pop the best node from OPENs\\
            Add $\tau_k^*$ to $N_{L, L}$\\
            \If{$k=K$}{
                \Return $N_{L, L}$ 
            }
            Index edges in $\tau_k^*$ as $\mathcal{L}=\{e_1, e_2, ..., e_m\}$\\
            \ForEach{$l \in \mathcal{L}$}{
                $(c_{k+1,l}, \tau_{k+1,l}^*, I_{e,k+1,l}, E_{e,k+1,l}) \gets$ new node\\
                $I_{e,k+1,l} \gets I_{e,k} \cup \{e_1,...,e_{l-1}\}$\\
                $E_{e,k+1,l} \gets E_{e,k} \cup \{e_{l}\}$\\
                $\tau_{k+1,l}^* \gets SolveRTSP(M_C, I_{e,k+1,l}, E_{e,k+1,l})$\\
                $c_{k+1,l} \gets \text{cost}(\tau_{k+1,l}^*)$\\
                \If{$c_{k+1,l} < \infty$ }{
                    Insert $(c_{k+1,l}, \tau_{k+1,l}^*, I_{e,k+1,l}, E_{e,k+1,l})$ to OPENs\\
                }
            }
        }
    \Return false
\end{algorithm}

\subsubsection{$K$-best joint task sequencing}
After obtaining the single agent $K$-best task sequence, we can proceed to solve the joint $K$-best task sequencing problem. The overall idea remains to use an A*-like best first manner. Algorithm~\ref{algo:kbestjoint} demonstrates the overall process. We initialize the set $N_{L,H}$ to store the $K$-best solutions. $P_{L,k}$ is the priority set of each $\tau^i$ in $\mathcal{\Tau}_k^*$. For example, $\mathcal{\Tau}_1^* = \{ (\tau^i)_1^* \mid i \in [N] \}$, so $P_{L,1} = \{1, 1, \ldots, 1\}$ (a total of $N$ ones). After obtaining $\mathcal{\Tau}_1^*$ and its cost $C_1$ (lines 2-10), we insert the node $(C_1, \mathcal{\Tau}_1^*, P_{L,1})$ into the priority queue OPENj (where 'j' stands for 'joint') (line 11). For the $k$-th iteration (lines 12-28), we pop the node $(C_k, \mathcal{\Tau}_k^*, P_{L,k})$ with the least $C$, and add $\mathcal{\Tau}_k^*$ to $N_{L,H}$ (lines 13-14). If $k$ is equal to $K$, the algorithm terminates and returns $N_{L,H}$ (lines 15-17). Otherwise, for the $l$-th priority number in $P_{L,k}$, we generate a new node with its priority set denoted as $P_{L,k+1,l}$. The only difference between $P_{L,k+1,l}$ and $P_{L,k}$ is that the $l$-th priority number in $P_{L,k+1,l}$ is increased by 1 compared to the corresponding number in $P_{L,k}$ (lines 19-20). Based on the new $P_{L,k+1,l}$, we can obtain $\mathcal{\Tau}_{k+1,l}^*$ and $C_{k+1,l}$, and add it to OPENj if feasible (lines 21-26).

We use an example with 3 agents to explain Algorithm~\ref{algo:kbestjoint} more clearly, as shown in Fig \ref{fig:k-best-ts}. Let the single agent $K$-best task sequencing for $a^i$ be $\{ (\tau^i)_1^*, (\tau^i)_2^*, \ldots \}$. The optimal solution for the joint task sequencing, $\mathcal{\Tau}_1^*$, is obviously $\{ (\tau^1)_1^*, (\tau^2)_1^*, (\tau^3)_1^* \}$, corresponding to $P_{L,1} = \{ 1, 1, 1 \}$. The second-best solution $\mathcal{\Tau}_2^*$ has three possibilities: $\mathcal{\Tau}_{2,1}^* = \{ (\tau^1)_2^*, (\tau^2)_1^*, (\tau^3)_1^* \}$ ($P_{L,2,1} = \{ 2, 1, 1 \}$), $\mathcal{\Tau}_{2,2}^* = \{ (\tau^1)_1^*, (\tau^2)_2^*, (\tau^3)_1^* \}$ ($P_{L,2,2} = \{ 1, 2, 1 \}$), and $\mathcal{\Tau}_{2,3}^* = \{ (\tau^1)_1^*, (\tau^2)_1^*, (\tau^3)_2^* \}$ ($P_{L,2,3} = \{ 1, 1, 2 \}$). Thus, these three nodes are added to OPENj. Assuming $\mathcal{\Tau}_2^* = \mathcal{\Tau}_{2,1}^*$, we pop it from OPENj, while $\mathcal{\Tau}_{2,2}^*$ and $\mathcal{\Tau}_{2,3}^*$ remain in OPENj. At this point, $\mathcal{\Tau}_3^*$ must be one of $\mathcal{\Tau}_{2,2}^*$, $\mathcal{\Tau}_{2,3}^*$, and the extensions of $\mathcal{\Tau}_{2,1}^*$: $\mathcal{\Tau}_{3,1}^*$ ($P_{L,3,1} = \{ 3, 1, 1 \}$), $\mathcal{\Tau}_{3,2}^*$ ($P_{L,3,2} = \{ 2, 2, 1 \}$), and $\mathcal{\Tau}_{3,3}^*$ ($P_{L,3,3} = \{ 2, 1, 2 \}$). Therefore, $\mathcal{\Tau}_{3,1}^*$, $\mathcal{\Tau}_{3,2}^*$, and $\mathcal{\Tau}_{3,3}^*$ are added to OPENj, and in the next iteration, the best one in OPENj is popped. By repeating this process, Algorithm~\ref{algo:kbestjoint} can output the $K$-best joint task sequencing.

\begin{figure}[htpb]
    \centering
    \includegraphics[width=1\linewidth]{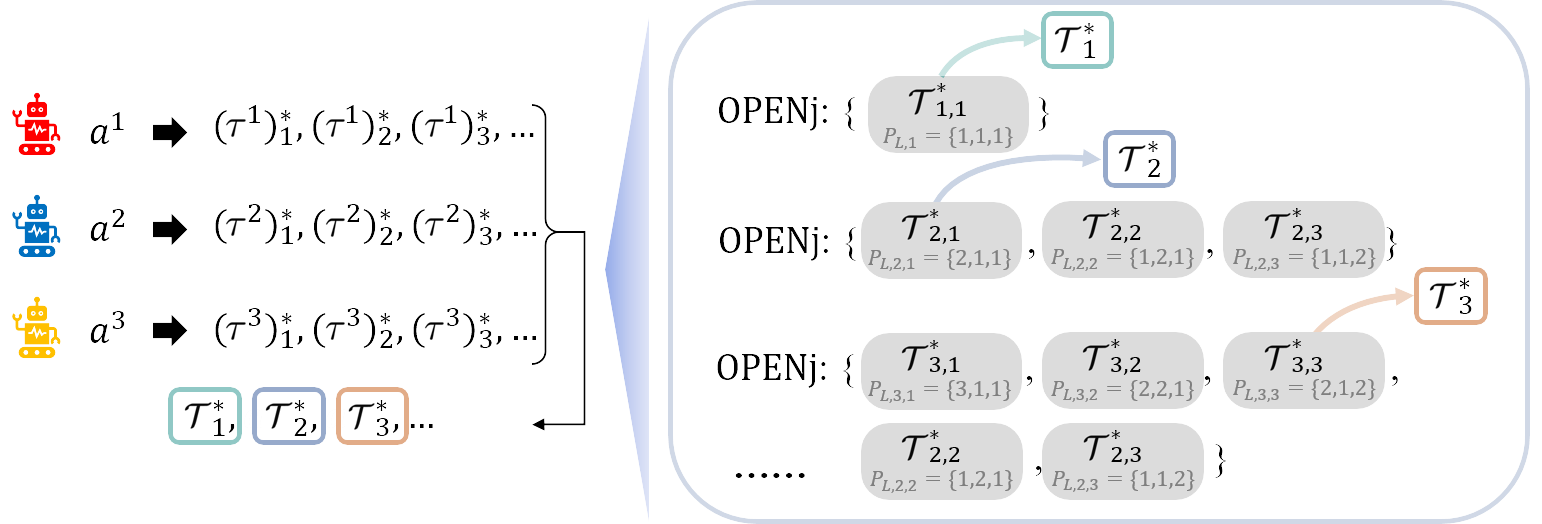} 
    \vspace{-5mm}
    \caption{The illustration of Algorithm \ref{algo:kbestjoint}. In each iteration, the best node from OPENj is popped, followed by adding its extensions into OPENj. By iterating this process, Algorithm \ref{algo:kbestjoint} is capable of producing the first, second, ..., $k$-th optimal solutions for joint task sequencing.}
    \label{fig:k-best-ts}
\end{figure}

\begin{algorithm}[htpb]
    \caption{$K$-best joint task sequencing}\label{algo:kbestjoint}
        \SetKwInOut{Input}{Input}

        \Input{$K$, graph $G$, agents set $A$ and number $N$, vertices set $V_s, V_d, V_t$, assignment $f$}
        $N_{L, H} \gets \emptyset$\\
        $P_{L, 1} \gets \emptyset$\\
        \ForEach{$i \in [N]$}{
            $G^i \gets GenerateSingleGraph(G, i, V_s, V_d, V_t, f)$\\
            $N_{L,L} \gets SingleAgentKBest(G^i, 1)$\\
            ${(\tau^i)}_1^* \gets N_{L,L}(1)$\\
            Add (1) to $P_{L,1}$\\
        }
        $\mathcal{\Tau}_1^* = \{{(\tau^i)}_1^* \mid i \in [N]\}$\\
        $C_1 \gets \text{cost}\mathcal{\Tau}_1^*$\\
        Insert $\left( C_1, \mathcal{\Tau}_1^*, P_{L, 1} \right)$ to OPENj \\
        
        \While{\textnormal{OPENj} not empty}{
            $(C_k, \mathcal{\Tau}_k^*, P_{L, k}) \gets$ pop the best node from OPENj\\
            Add $\mathcal{\Tau}_k^*$ to $N_{L, H}$\\
            \If{$k=K$}{
                \Return $N_{L, H}$ 
            }
            \ForEach{$l \in [N]$}{
                $(C_{k+1,l}, \mathcal{\Tau}_{k+1,l}^*, P_{L,k+1,l}) \gets$ new node\\
                $P_{L,k+1,l}(l) \gets P_{L,k}(l)+1$\\
                $N_{L,L} \gets SingleAgentKBest(G^l, P_{L,k+1,l}(l))$\\
                $\mathcal{\Tau}_{k+1,l}^*(l) \gets N_{L,L}(l)$\\
                $C_{k+1,l} \gets \text{cost}(\mathcal{\Tau}_{k+1,l}^*)$\\
                \If{$C_{k+1,l} < \infty$ }{
                    Insert $(C_{k+1,l}, \mathcal{\Tau}_{k+1,l}^*, P_{L,k+1,l})$ to OPENj\\
                }
            }
        }
    \Return false
\end{algorithm}

Through the above algorithms, we obtain the $K$-best solutions for the jTSP problem. Each joint task sequencing in the results will serve as a root node to generate the search forest of CTS-CBS.

\subsection{Necessity Checking for New Root Nodes}\label{sub:nece-check}
Next, we will introduce another key issue in CTS-CBS: when to generate a new search tree. We address this issue through a necessity checking process, while ensuring the optimality / bounded suboptimality of the algorithm. First, we introduce the definition of bounded suboptimality:
\begin{definition}
Bounded suboptimality refers to the property of an algorithm that guarantees to return a solution $\Pi$ with $\text{cost}(\Pi) \leq (1+\omega) \cdot \text{cost}(\Pi^*)$, where $\Pi^*$ is the unknown optimal solution. $\omega$ is the user-defined boundary parameter.
\end{definition}

The process of necessity checking for a new search tree is shown in Algorithm~\ref{algo:nece-check}, with node $P_p$, priority queue OPEN, and boundary parameter $\omega$ as input. We denote the already generated joint task sequences as $\{\mathcal{\Tau}_1^*, \mathcal{\Tau}_2^*,...,\mathcal{\Tau}_R^*\}$ (line 1). New search trees need to be generated under one of two circumstances: First, when OPEN is empty after popping node $P_p$, which indicates that all nodes in the current search trees have been explored; second, when the cost $g$ of $P_p$ exceeds $(1+\omega)$ times the cost of $\mathcal{\Tau}_R^*$ (lines 2-3). Note that in $JointKBest$, we omit other input variables besides $K$ for simplicity of expression. Then, if the next best joint task sequence $\mathcal{\Tau}_{R+1}^*$ exists (line 4), a new root node $P_{nr}$ is generated (lines 5-8). If the cost of $P_p$ is less than the cost of $P_{nr}$, $P_p$ is returned and $P_{nr}$ is inserted to OPEN (lines 9-12). Otherwise, $P_{nr}$ is returned and $P_p$ is inserted to OPEN (lines 13-14). If neither of the above two conditions is met, or if they are met but the next optimal solution $\mathcal{\Tau}_{R+1}^*$ does not exist, then node $P_p$ is returned (line 17).

If the cost $g$ of $P_p$ is less than the cost of the $\mathcal{\Tau}_R^*$, the joint task sequence of current last generated search tree $CT_R$, then there is no need to generate a new search tree, and node $P_p$ is returned (lines 2-4). Otherwise, a new search tree corresponding to the joint task sequence $\mathcal{\Tau}_{R+1}^*$ needs to be generated using the algorithm in section~\ref{sub:Kbest} (line 5). Subsequently, based on $\mathcal{\Tau}_{R+1}^*$, a new root node $P_{nr}$ is generated (lines 6-9). If the cost of $P_p$ is less than the cost of $P_{nr}$, $P_p$ is returned and $P_{nr}$ is inserted to OPEN. Otherwise, $P_{nr}$ is returned and $P_p$ is inserted to OPEN.

\vspace{-1.5mm}
\begin{algorithm}[htbp]
    \caption{Necessity checking for new search tree}\label{algo:nece-check}
        \SetKwInOut{Input}{Input}

        \Input{$P_p$, OPEN, $\omega$}
        denote the already generated joint task sequences as $\{\mathcal{\Tau}_1^*, \mathcal{\Tau}_2^*,...,\mathcal{\Tau}_R^*\}$\\
        \If{$\text{OPEN}=\emptyset \lor P_{p}.g > (1+\omega)\text{cost}(\mathcal{\Tau}_R^*)$}{
            $\mathcal{\Tau}_{R+1}^* \gets JointKBest(K=R+1)$\\
            \If{$\mathcal{\Tau}_{R+1}^* \neq false$}{
                $P_{nr} \gets$ new root node\\
                $P_{nr}.\Omega \gets \emptyset$\\
                $P_{nr}.\Pi \gets LowLevelSearch(\mathcal{\Tau}_{R+1}^*, P_{nr}.\Omega)$\\
                $P_{nr}.g \gets \text{cost}(P_{nr}.\Pi)$\\
                \If{$P_{p}.g < P_{nr}.g$}{
                    Insert $P_{nr}$ to OPEN\\
                    \Return $P_p$
                }
                Insert $P_{p}$ to OPEN\\
                \Return $P_{nr}$
            }
        }
        \Return $P_p$
\end{algorithm}

\vspace{-2mm}
The boundary parameter $\omega$ constrains the cost of the solution output by the algorithm to be within $\omega$ times the cost of the optimal solution. If $\omega=0$, then once the cost of the current node $P_p$ exceeds the cost of the root node of the last generated search tree, a new search tree with the next best $\mathcal{\Tau}^*$ needs to be generated. This ensures the solution is optimal, but the computational burden is the highest because the algorithm needs to constantly solve for the $K$-best joint task sequencing, which is time-consuming. If $\omega = \infty$, CTS-CBS will generate only one search tree with the best joint task sequence $\mathcal{\Tau}_1^*$ (because line 2 in Algorithm~\ref{algo:nece-check} will always be satisfied), and CTS-CBS will look for a collision-free joint path that follows $\mathcal{\Tau}_1^*$. In this case, the completeness and optimality of the algorithm cannot be guaranteed, but the computational efficiency is relatively high (if a feasible solution can be found). If $\omega$ is between 0 and $\infty$, the algorithm can balance optimality and efficiency, where the completeness of the algorithm is still guaranteed, and it exhibits bounded suboptimality (proof in Section~\ref{Sec:TheoAnal}).

\subsection{Implementation Details}

\subsubsection{$K$-best joint task sequencing}
Solving the $K$-best joint task sequencing problem is time-consuming. In practical applications, we improve efficiency through the following methods: In the upper layer, we keep recording OPENj and only calculate the next best joint task sequence when necessary to generate a new search tree. By this way, when $JointKBest(K=R+1)$ (Algorithm~\ref{algo:kbestjoint}) is called in Algorithm~\ref{algo:nece-check}, we do not need to start from scratch but can continue from $K=R$. In the lower layer, we use $N_{L,L}$ to record and store the single agent $K$-best sequences. Thus, each time Algorithm~\ref{algo:singleagentkbest} is called, we can first check if it has already been computed and reuse the previously calculated results.

\subsubsection{Tools to solve single agent rTSP}
Several well-established solvers can be employed to solve the rTSP. For example, we can handle the constraints in rTSP by setting the weights of edges in $I_e$ to 0 and assigning a very large number to the weights of edges in $E_e$, thereby treating it as an ordinary TSP. This allows us to use relatively high-efficient solvers. In this study, we use Gurobi as the rTSP solver due to its simplicity in implementation, good generalization, and satisfied performance. Although Gurobi does not guarantee an optimal solution, it often provides optimal solutions in many scenarios, especially in relatively simple cases. Additionally, we have modularized the rTSP solver to facilitate the replacement of the existing module with a more performant rTSP solver in the future, should one become available.

\subsubsection{Low-level search of CTS-CBS} 
For the low-level search of CTS-CBS, A*-like algorithms can be used. We use SIPP \cite{phillips2011sipp} because it has been demonstrated in \cite{ren2023cbss} that SIPP runs faster than space-temporal A* in the low-level modules of CBS-like algorithms.

\section{Theoretical Analysis} \label{Sec:TheoAnal}


Now, we will prove the completeness and bounded suboptimality of the CTS-CBS algorithm.

\begin{theorem}
The CTS-CBS algorithm is complete when $\omega < \infty$, which means that if the input CTS-MAPF instance has a solution, CTS-CBS will always find one.
\end{theorem}

\begin{proof}
As previously mentioned, node expansion in CTS-CBS occurs in only two forms: generating a new root node or expanding nodes within an existing $CT$. Note that if the CTS-MAPF problem has a solution $\Pi$, it must follow a specific joint task sequence $\mathcal{T}_j^*$, meaning that the node corresponding to solution $\Pi$ must be within the search tree $CT_j$ corresponding to $\mathcal{T}_j^*$. Additionally, for a CTS-MAPF problem, the number of possible task sequences is finite, and within any $CT$, the number of nodes with a cost below a certain value is also finite (until all nodes in that search tree are exhausted). CTS-CBS generates new $CTs$ in a best-first manner and sequentially inserts all nodes under these trees into OPEN. For a solution $\Pi$ with finite cost, CTS-CBS is guaranteed to expand the $CT$ containing node with joint path $\Pi$, within a finite number of expansions and find $\Pi$ during the process of popping all nodes from OPEN. Therefore, CTS-CBS is guaranteed to find a solution to the problem within finite time.
\end{proof}

\begin{theorem}
The CTS-CBS algorithm is $\omega$-bounded sub-optimal when $\omega < \infty$. For a CTS-MAPF instance, if the cost of the optimal solution is $g^*$, the CTS-CBS algorithm can guarantee a solution $\Pi$ such that $g(\Pi) \leq (1+\omega) g^*$.
\end{theorem}

\begin{proof}
Consider the optimal solution corresponding to the node $P_B^*(\Omega_B^*, \Pi_B^*, g_B^*)$, residing in $CT_B$ with the joint task sequence $\mathcal{\Tau}_B^*$. The node returned by CTS-CBS is $P_S^*(\Omega_S^*, \Pi_S^*, g_S^*)$, associated with $CT_S$ and the joint task sequence $\mathcal{\Tau}_S^*$. Since search trees are generated in a priority order, there are only two possibilities for $CT_B$ and $CT_S$: 1) $CT_S$ is generated before $CT_B$; 2) $CT_S$ is generated after $CT_B$ (including the case where $CT_S$ and $CT_B$ are the same tree). 

In the first case, it must be that $g_S^* \leq (1 + \omega) \text{cost}(\mathcal{\Tau}_B^*)$; otherwise, the search tree $CT_B$ corresponding to $\mathcal{\Tau}_B^*$ would be generated. Furthermore, $g_B^* \geq \text{cost}(\mathcal{\Tau}_B^*)$ because the costs of nodes within the same search tree are non-decreasing. Thus, we have: $g_S^* \leq (1 + \omega) \text{cost}(\mathcal{\Tau}_B^*) \leq (1 + \omega) g_B^*$.

In the second case, since $CT_B$ has already been generated, all nodes in $CT_B$ will be sequentially added to OPEN until the algorithm returns a solution. Therefore, $P_B^*$ will also be added to OPEN. Given that $P_B^*(\Omega_B^*, \Pi_B^*, g_B^*)$ is the optimal solution, $g_B^*$ is the smallest cost among all feasible solutions, and since OPEN always pops the node with the smallest cost, the solution found $P_S^*$ must be $P_B^*$, meaning $g_S^* = g_B^*$. 

In summary, the bounded suboptimality of the CTS-CBS algorithm has been proved.
\end{proof}

\section{Experiments} \label{Sec:Experiments}
In this section, we first introduce the settings of the CTS-CBS algorithm, the datasets used for experiments, and the four baselines for evaluation. The performance of the CTS-CBS algorithm is then compared against baselines. Subsequently, the impact of the parameter $\omega$ on the efficiency and quality will be investigated. Finally, the CTS-CBS algorithm is deployed on practical robots to validate its applicability in real-world robotic scenarios.

\vspace{-3mm}
\subsection{Settings, Datasets and Baselines}

As previously mentioned, Gurobi is utilized as the rTSP solver, while SIPP is employed as the low-level search method. For each instance in the experiments, a time limit of \textbf{three minutes} is set. This implies that if the algorithm fails to produce a solution within three minutes, it is considered a failure. The CTS-CBS algorithm is implemented in C++, and all experiments are conducted on a computer equipped with an Intel Core i7-12700KF CPU, an NVIDIA GeForce RTX 3080 GPU, and 32GB of RAM, running on a Linux system.

Datasets play a crucial role in evaluating the performance of algorithms. Since there are currently no publicly available datasets specifically designed for the CTS-MAPF problem, we adapted the MAPF benchmarks provided by Moving AI \cite{stern2019mapf}. We selected three maps with varying levels of difficulty: \texttt{Empty}, \texttt{Random}, and \texttt{Room}, as shown in Fig. 4. For each map, the benchmarks provide 50 scenarios, each containing about 100 pieces of agent data with specified start and goal positions. We used the first 20 pieces as agent data for the CTS-MAPF problem and took the remaining start positions as task data. Through this adaptation, we obtained 50 scenarios for each map, all of which are suitable for testing the effectiveness of algorithms designed for the CTS-MAPF problem.

In the aforementioned dataset (referred to as the \texttt{CTS-MAPF} dataset), each agent is assigned a unique goal, along with a sequence of task points. The agent must start from its initial position, visit all tasks in a specified order, and finally reach its designated goal. Based on this dataset, we also provide an alternative dataset where the goal is treated as an ordinary task, meaning the agent's goal does not necessarily have to be the destination. Under this condition, the CTS-MAPF problem transforms to the MG-MAPF problem, and we thus refer to this dataset as the \texttt{MG-MAPF} dataset. In this second dataset, only minor modifications to the CTS-CBS algorithm are needed to facilitate comparison with other algorithms. In Section \ref{Sec:Experiments-B}, we employ the \texttt{MG-MAPF} dataset to validate the performance of the CTS-CBS algorithm against baselines. In Sections \ref{Sec:Experiments-C}, we utilize the \texttt{CTS-MAPF} dataset to conduct an in-depth analysis of the CTS-CBS algorithm. 

\vspace{-3mm}
\begin{figure}[htbp]
  \centering
  \includegraphics[width=0.95\linewidth]{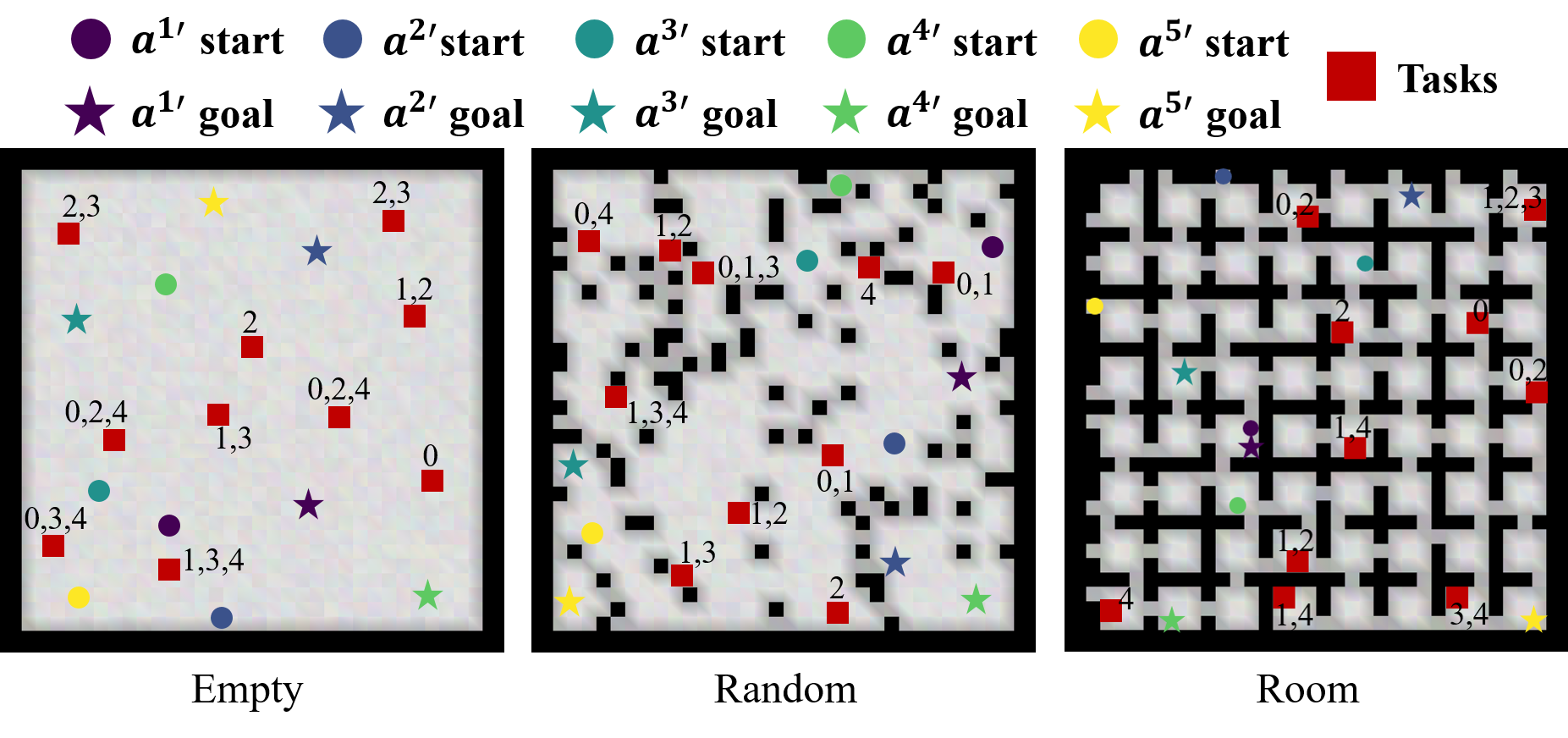} 
    \vspace{-3mm}
  \caption{The maps used in experiments. The images also visualize one specific instance on different maps when the number of agents is 5 and the number of tasks is 10. Circles of different colors represent the agents' start locations, while stars denote their goal positions. Red squares represent tasks, with the numbers beside them indicating which agents need to visit them.}
  \label{fig:maps}
\end{figure}

We select four baselines in Section \ref{Sec:Experiments-B} for comparison. 
\begin{enumerate}
\item The \textbf{first} one is \textbf{Optimal-CBS}, which uses the CBS framework combined with a low-level designed A* method capable of finding the shortest path for the agent to visit all tasks. Optimal-CBS has optimality guarantee.

\item The \textbf{second} one is \textbf{HCBS} proposed in \cite{surynek2021multi}. HCBS operates on three levels. The top level is the CBS framework, the middle is an A*-like algorithm to solve the visiting order of agents, and the bottom is a traditional A* algorithm for single-agent path planning.

\item The \textbf{third} one is \textbf{MGCBS} proposed in \cite{tang2024mgcbs}. MGCBS decouples the goal-safe interval visiting order search from single-agent path planning and employs a Time-Interval Space Forest to enhance efficiency. HCBS and MGCBS are considered the current state-of-the-art (SOTA) methods for the MG-MAPF problem.

\item The \textbf{fourth} one is the sequential method mentioned earlier, which corresponds to the case where the parameter $\omega$ in CTS-CBS is infinite. It can also be seen as a greedy method. For simplicity, this method will be referred to as \textbf{S(equential)-CBS} in the following text.

\end{enumerate}

In all experiments, the number of agents $N$ is from set $\{5, 10, 20\}$, and the number of tasks $M$ is from set $\{10, 20, 30, 40, 50\}$. Each task is randomly assigned to 1,2, or 3 agents. Fig. 4 also illustrates the configuration of a specific instance when $N=5$ and $M=10$ across different maps. Circles of different colors represent the agents' starts, while stars denote their goals. Red squares represent tasks, with the numbers beside them indicating agents that need to visit them.

\vspace{-3mm}
\begin{figure*}[t]
 \begin{center}    
      \includegraphics[width=0.95\linewidth]{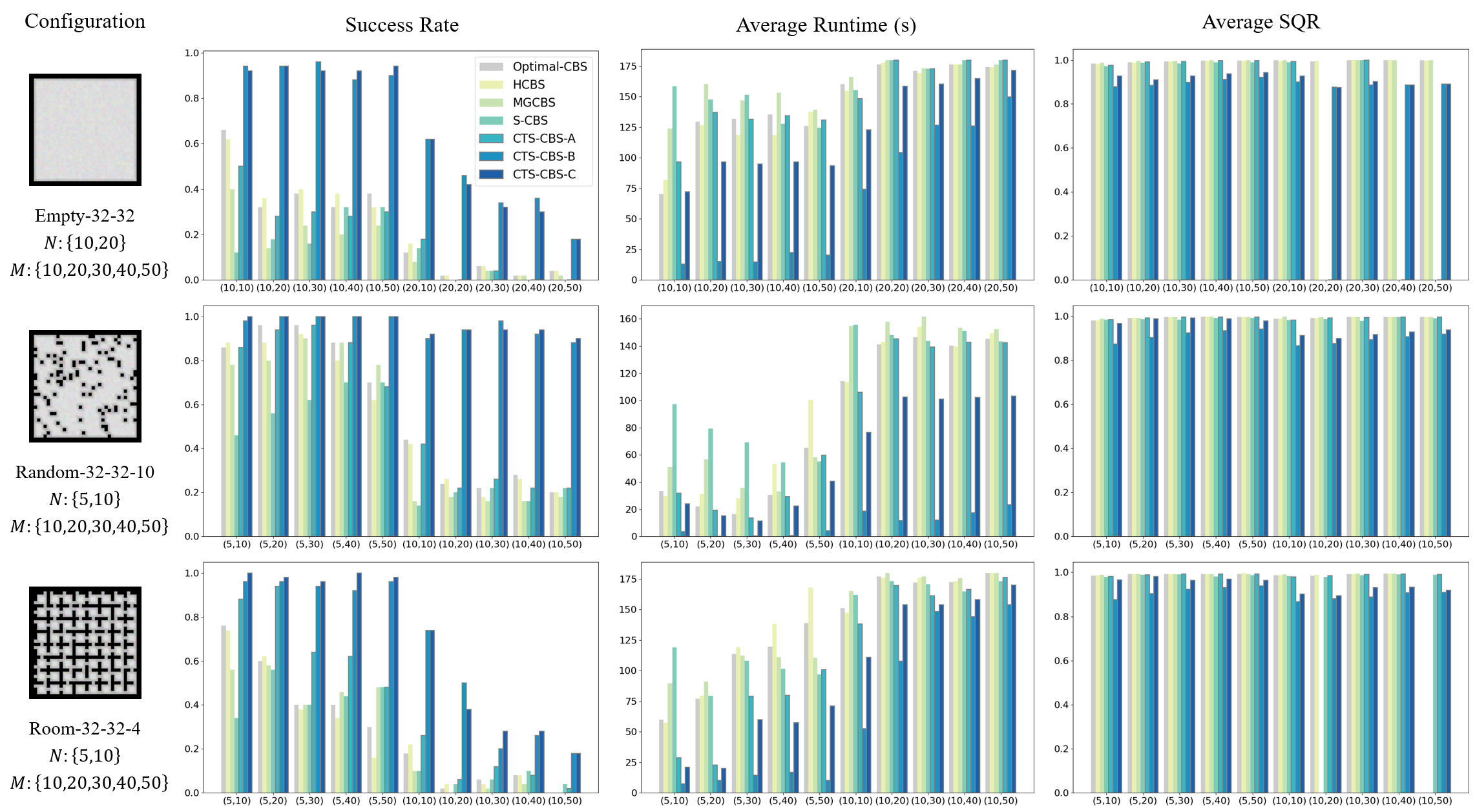}
 \end{center}
 \vspace{-3mm}
 \caption{The performance of the CTS-CBS adaptations and baseline algorithms on MG-MAPF problems under varying difficulty (different maps, agent numbers, and task numbers). The CTS-CBS adaptations significantly outperform all baseline algorithms in terms of success rate (by up to 20 times larger) and runtime (by up to 100 times faster), with less than a 10\% sacrifice in solution quality.}
 \label{fig:comparison}
 \vspace{-3mm}
\end{figure*}

\vspace{-3mm}
\subsection{Comparative Results and Analysis}\label{Sec:Experiments-B}
In this section, we evaluate the performance of various algorithms on the \texttt{MG-MAPF} dataset. This is motivated by the fact that, while there are existing methods for the MG-MAPF problem, there is currently no dedicated research addressing the CTS-MAPF problem. However, our algorithm, CTS-CBS, is not only capable of solving the CTS-MAPF problem but can also be adapted to solve the MG-MAPF problem with minor modifications. Therefore, evaluating the performance of CTS-CBS against baselines on the \texttt{MG-MAPF} dataset can partially demonstrate the generalization of the CTS-CBS algorithm. 

There are several ways to adapt the CTS-CBS algorithm for the MG-MAPF problem. To comprehensively compare the algorithm's performance, we propose three adaptations:

\begin{enumerate}
    \item \textbf{CTS-CBS-A:} In the rTSP-solving phase, the terminal constraint is removed, allowing the agent to visit all tasks and the final destination in any order. CTS-CBS-A is capable of providing either the optimal or near-optimal solution for the MG-MAPF problem.
    
    \item \textbf{CTS-CBS-B:} The MG-MAPF problem is transformed back into CTS-MAPF by assigning one of the tasks as the destination for each agent, thereby enabling the use of the CTS-CBS algorithm. Although this reduces the solution space of the MG-MAPF problem which may affect the optimality of the solution, it could improve the success rate and efficiency of finding a feasible solution.
    
    \item \textbf{CTS-CBS-C:} This adaptation combines the above two approaches. The combination can be implemented in various ways; for example, one could use a sequential approach, where CTS-CBS-A is applied first, and if no solution is found within a certain time, CTS-CBS-B is employed for the remaining time. Alternatively, inspired by the anytime algorithm, CTS-CBS-B can be used initially to obtain a solution, followed by CTS-CBS-A to refine it towards optimality within the remaining time. In this paper, we use the first combination method, allocating two-thirds of the time limit to CTS-CBS-A.
\end{enumerate}

We compare the performance of various algorithms from three perspectives: success rate, efficiency (runtime), and solution quality. Below are the definitions of these three metrics:

\begin{enumerate}
    \item \textbf{Success Rate:} The proportion of successful instances in a given scenario. A higher success rate indicates better algorithm performance.
    
    \item \textbf{Runtime:} The average time taken to solve all instances in a given scenario. For failed instances, the runtime is considered to be the time limit. A smaller runtime indicates higher algorithm efficiency.
    
    \item \textbf{Solution Quality:} This metric reflects the gap between the solution provided by the algorithm and the optimal solution for the problem. Since the optimal solution cannot be obtained for all instances, we use the lower bound of the optimal solution as a proxy. The cost of the first root node  $P_{\text{root1}}$  in CTS-CBS-A, denoted as $\text{cost}_{\text{lb}}$, serves as the lower bound of the optimal solution. For all successful instances with a solution cost $\text{cost}_{\text{sol}}$, we define the Solution Quality Ratio (SQR) as $\text{SQR} = \text{cost}_{\text{lb}} / \text{cost}_{\text{sol}}$. The average SQR across all successful instances indicates the solution quality in a given scenario, with higher average SQR for better optimality of the algorithm.
\end{enumerate}

In this section, we set the parameter $\omega$ of CTS-CBS to 0.01, with results shown in Fig. \ref{fig:comparison}. Below, we introduce the performance of three adaptations:

\subsubsection{CTS-CBS-A}
CTS-CBS-A can generate near-optimal solutions (with a cost no greater than $(1 + \omega)$ times the optimal solution). Therefore, compared with baselines aiming to generate optimal solutions, the average SQR of CTS-CBS-A is nearly identical. Another noticeable trend is that as the map difficulty increases (from \texttt{Empty} to \texttt{Random} to \texttt{Room}), the performance advantage of CTS-CBS-A over the baselines becomes more pronounced. Specifically, on the \texttt{Room} map, CTS-CBS-A achieves an obviously shorter runtime and higher success rate than the baselines, while maintaining similar solution quality.

\subsubsection{CTS-CBS-B}
CTS-CBS-B reduces the solution space by transforming the MG-MAPF problem into CTS-MAPF, sacrificing optimality. However, the solution efficiency of CTS-CBS-B is significantly improved. In both simple and difficult scenarios, CTS-CBS-B achieves a success rate several times higher than that of the baselines, with the runtime reduced by one to two orders of magnitude in some cases. Moreover, the results show that the average SQR of CTS-CBS-B is above 0.85, which generally indicates that the solution quality remains acceptable. One possible explanation for why reducing the solution space does not significantly affect solution quality and success rate is that solutions to the MG-MAPF problem are not sparse, so the probability of finding a feasible solution within a subspace remains high.

\subsubsection{CTS-CBS-C}
CTS-CBS-C is a combination of the previous two adaptations, so its performance is expected to lie between them. The results show that in most cases, the success rate of CTS-CBS-C is similar to that of CTS-CBS-B, significantly higher than the baselines. In terms of runtime, CTS-CBS-C falls between CTS-CBS-A and CTS-CBS-B but is still significantly lower than the baselines. As for solution quality, the harmonization of CTS-CBS-C is most evident. As the problem difficulty increases (i.e., as the agent number and task number grow), the average SQR of CTS-CBS-C gradually transitions from that of CTS-CBS-A to that of CTS-CBS-B. This is because, in simpler scenarios, the solution of CTS-CBS-C is mostly generated by CTS-CBS-A. However, in more challenging scenarios, CTS-CBS-A fails to find a solution within the time limit, and thus the solutions of CTS-CBS-C are predominantly generated by CTS-CBS-B.

In summary, the CTS-CBS adaptations significantly outperform all baseline algorithms in terms of success rate (by up to 20 times larger) and runtime (by up to 100 times faster), with less than a 10\% sacrifice in solution quality.

\vspace{-3mm}
\subsection{Influence of Sub-optimality Bound} \label{Sec:Experiments-C}
Next, we aim to investigate the influence of the sub-optimality bound parameter $\omega$. Using the \texttt{Room} map as an example and setting $N=5$, we explore three values for $\omega$: 0, 0.01, and 0.1. Additionally, we include the case where $\omega$ is set to infinity (i.e., the forth baseline) in our analysis. In addition to the success rate, runtime, and solution quality proposed in the previous subsection, we introduce two additional metrics for comparison: the number of generated root nodes and the number of TSP solver calls. These metrics provide further insight into the effect of $\omega$. The results are presented in Fig. \ref{fig:omega-impact}, with each metric explained in detail below.

\begin{figure}[htbp]
  \centering
  \includegraphics[width=1\linewidth]{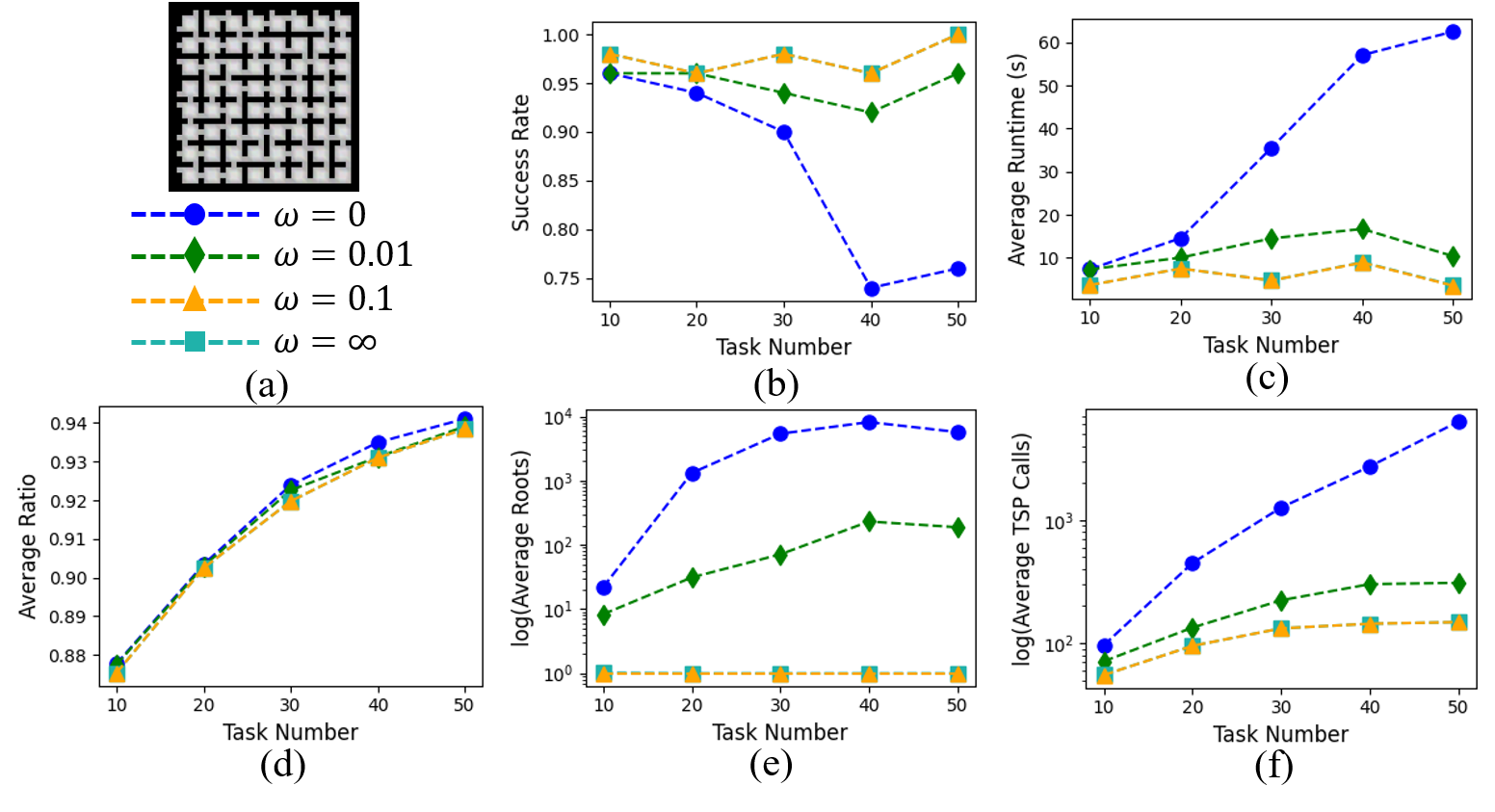} 
    \vspace{-3mm}
  \caption{The influence of the sub-optimality bound on the CTS-CBS algorithm. A clear trade-off emerges between solution efficiency (success rate (b) and runtime (c)) and solution quality (solution quality ratio (d)). This is because, as $\omega$ increases, CTS-CBS tends to focus on finding feasible solutions within existing $CT$s, delaying the generation of new root nodes (e), which in turn reduces the number of TSP solver calls (f).}
  \label{fig:omega-impact}
\end{figure}

\subsubsection{Success Rate}
As shown in Fig. \ref{fig:omega-impact}(b), increasing $\omega$ from 0 to 0.01 and from 0.01 to 0.1 both improve the success rate. This is because a feasible solution may be found after exploring multiple layers of nodes within each $CT$. When $\omega = 0$, CTS-CBS tends to generate a new root node after exploring only a few layers in a $CT$, moving on to explore the next $CT$. Since solving the TSP problem is more time-consuming than resolving conflicts within a single $CT$, it may result in a failure to find a feasible solution within the time limit. Therefore, when $\omega$ is increased to 0.01, it delays the generation of new root nodes, encouraging CTS-CBS to explore deeper layers within the current $CT$s, thus increasing the probability of finding a feasible solution. The trend further strengthens when $\omega$ is raised to 0.1, further improving the success rate. The results for $\omega = \infty$ are identical to those for $\omega = 0.1$, indicating that when $\omega = 0.1$, CTS-CBS has essentially degenerated into a sequential method and the algorithm only solves within the first $CT$.

\subsubsection{Runtime}
In CTS-CBS, once a feasible solution is found, it is immediately returned. As previously mentioned, increasing the value of $\omega$ enhances the probability of finding a feasible solution, thereby reducing the average runtime of the algorithm (Fig. \ref{fig:omega-impact}(c)). Considering both the success rate and runtime, it can be concluded that increasing $\omega$ improves the overall efficiency of the CTS-CBS algorithm.

\subsubsection{Solution Quality}
Fig. \ref{fig:omega-impact}(d) illustrates the variation in the average SQR under different values of $\omega$. It can be observed that as $\omega$ increases, the average SQR slightly decreases. This occurs because when $\omega$ is 0, if the CTS-CBS algorithm successfully finds a solution, it is guaranteed to be the optimal one (notably, the optimal solution typically does not equal the lower bound cost, $\text{cost}_{\text{lb}}$, so the solution quality ratio is also not 1 when $\omega$ is 0). As $\omega$ increases, the sub-optimality bound of the algorithm widens, leading to a decrease in solution quality.

\subsubsection{Number of Generated Root Nodes and TSP Solver Calls}
In summary, $\omega$ influences the trade-off between solution efficiency and quality. The reason for this balance can be seen from the number of root nodes generated (Fig. \ref{fig:omega-impact}(e)) and the number of TSP solver calls (Fig. \ref{fig:omega-impact}(f)) under different $\omega$ values. As $\omega$ increases, the number of root nodes generated decreases significantly, leading to a notable reduction in the number of TSP solver calls. When $\omega$ is set to 0.1, only one root node is generated, meaning CTS-CBS searches for a feasible solution within a single $CT$, thus improving solution efficiency at the expense of optimality.

\subsection{Practical Robot Tests}
In this study, we utilized a team of small physical robots to validate the CTS-CBS algorithm. The experimental platform employed is the toio robots (https://toio.io/). The toio robots communicate with the control center via Bluetooth and are guided to move on a dedicated mat (comprising a grid of 7x5) through the continuous issuance of coordinates.

During the experiment, we first established the experimental scenario by placing several obstacles on the dedicated mat and defining the starting and ending positions of each robot, as well as the locations of various tasks and their respective robot assignments. The configurations in the simulation environment were set to mirror those in the real-world environment.

Subsequently, the CTS-CBS algorithm was executed at the control center (in this case, a laptop) to compute the paths for each robot, which were then communicated to the robots. Fig. \ref{fig:real-robot} illustrates a set of snapshots of the demo. The robots were able to navigate according to their designated paths. The code for the experimental section was written in Python. The physical robot tests verified the applicability of the CTS-CBS algorithm.

\begin{figure*}[htpb]
 \begin{center}    
      \includegraphics[width=0.95\linewidth]{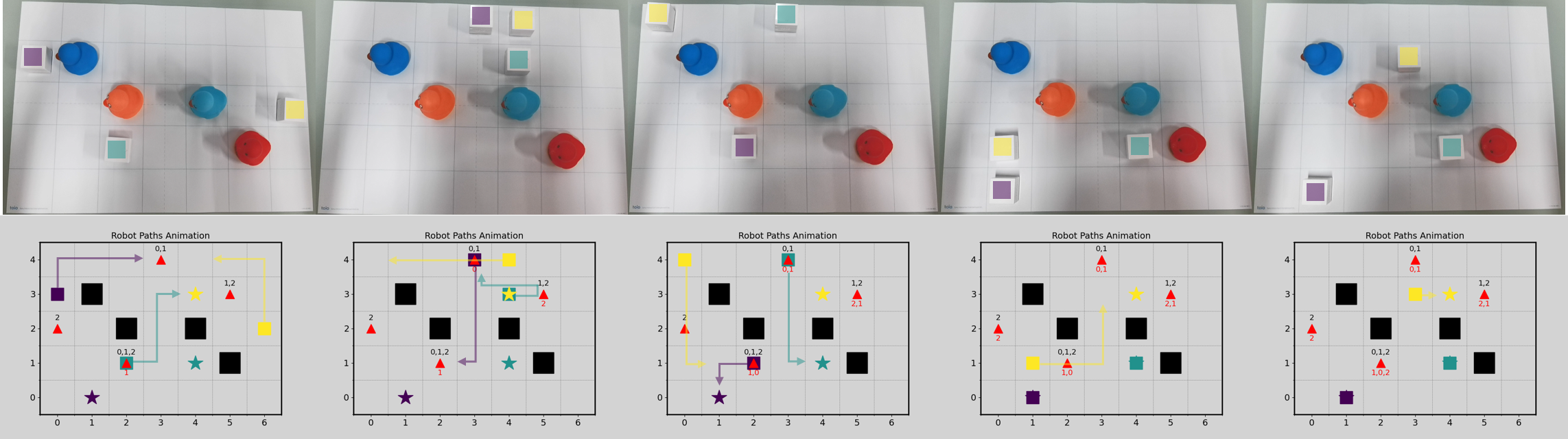}
 \end{center}

 \caption{The snapshots of the physical robot test demo. The purple, yellow, and green squares represent the robots, with corresponding colored stars indicating their respective destinations. The red triangles denote tasks, and the numbers above indicate which robots are assigned to each task. During the experiment, the red numbers below each task triangle represent the robots that have already reached that task.}
 \label{fig:real-robot}
 \vspace{-3mm}
\end{figure*}

\section{Conclusion} \label{Sec:Conc} 
This paper formulates a generalization problem of MAPF, called Collaborative Task Sequencing - Multi-Agent Pathfinding (CTS-MAPF), which requires agents to plan their task sequences while simultaneously planning collision-free paths. To address this problem, we develop a new approach called Collaborative Task Sequencing - Conflict-Based Search (CTS-CBS) and theoretically prove the completeness and optimality of the algorithm (or sub-optimality if a bounded parameter is given). We also create two datasets, \texttt{CTS-MAPF} and \texttt{MG-MAPF}, and conduct extensive experiments on them. The results indicate that the adaptations of CTS-CBS for MG-MAPF problem significantly outperform all baselines in terms of success rate (by up to 20 times larger) and runtime (by up to 100 times faster), with less than a 10\% sacrifice in solution quality. Additionally, CTS-CBS demonstrates greater flexibility; by varying the sub-optimality bound $\omega$, it can balance the trade-off between solution efficiency and quality. Finally, practical robot tests are conducted to verify the applicability of the algorithm in real-world robotic scenarios.

In future work, we plan to explore additional strategies to further enhance the efficiency of CTS-CBS, considering the kinematic constraints of agents to broaden the application scope, as well as extending our algorithm to dynamic and lifelong scenarios.

\bibliographystyle{IEEEtran}
\bibliography{bibfile/mybibfile}

\vspace{-30pt}
\begin{IEEEbiography}[{\includegraphics[width=1in,height=1.25in,clip,keepaspectratio]{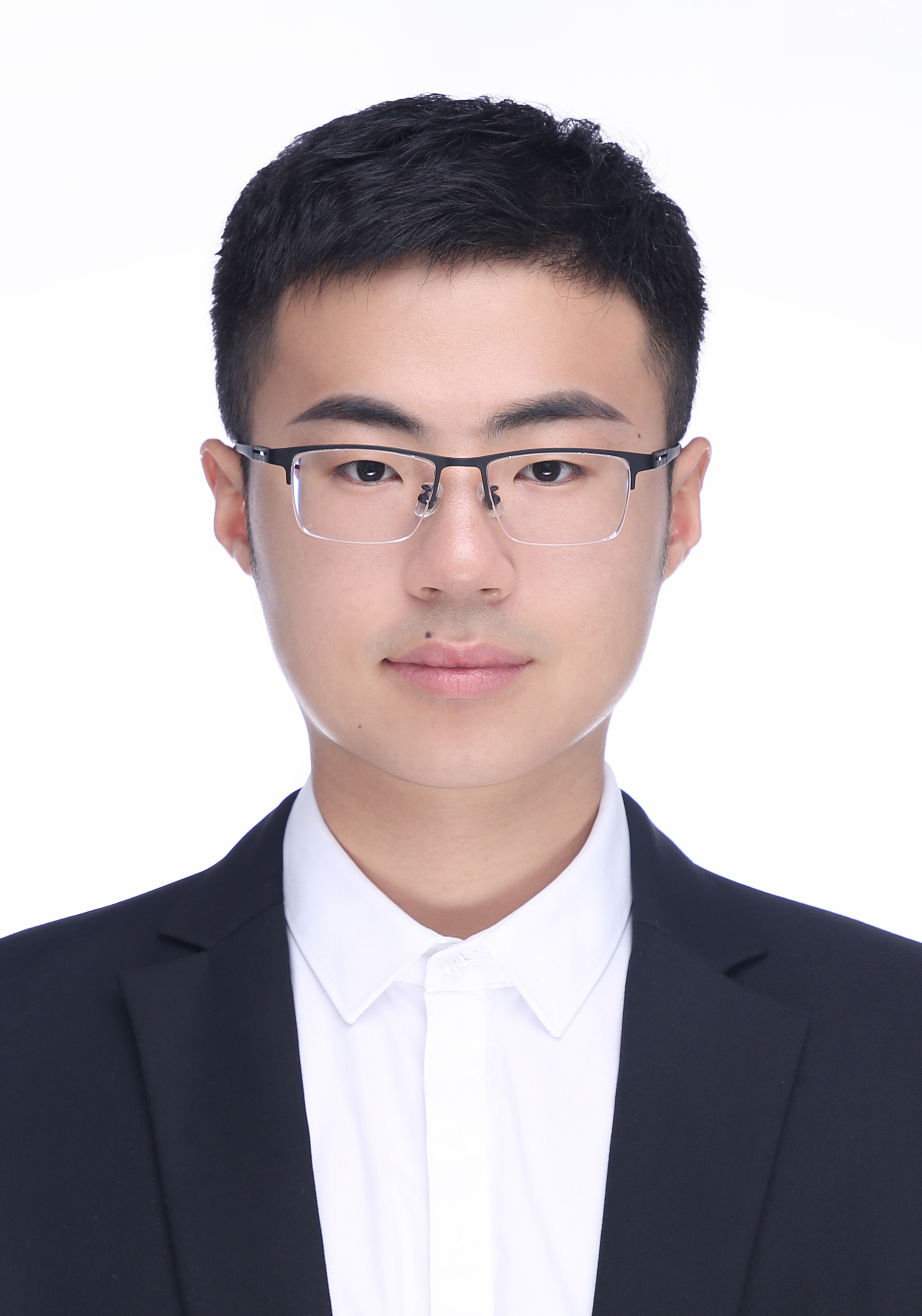}}]{Junkai~Jiang}
   received the B.E. degree from Tsinghua University, Beijing, China, in 2021, where he is currently pursuing the Ph.D. degree in mechanical engineering with the School of Vehicle and Mobility, Tsinghua University. His research interests include risk assessment, trajectory prediction, motion planning of intelligent vehicles, and multi-agent coordinate planning.
\end{IEEEbiography}

\vspace{-30pt}
\begin{IEEEbiography}[{\includegraphics[width=1in,height=1.25in,clip,keepaspectratio]{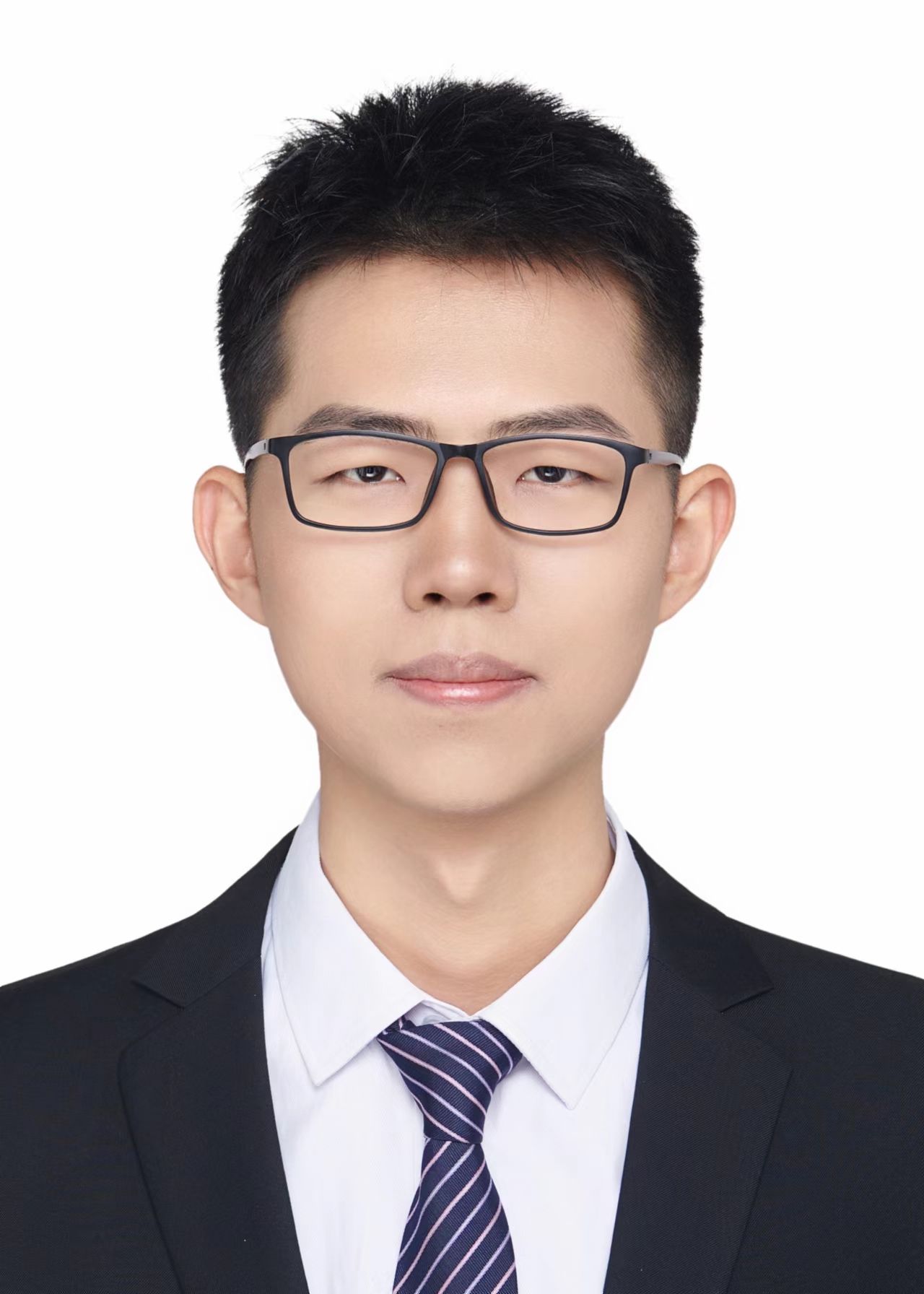}}]{Ruochen~Li}
   received the bachelor's degree in automotive engineering from School of Vehicle and Mobility, Tsinghua University, Beijing, China, in 2023. He is currently pursuing the Ph.D. degree in mechanical engineering with School of Vehicle and Mobility, Tsinghua University, Beijing, China. His research centered on decision-making of multiple intelligent and connected vehicles
\end{IEEEbiography}

\vspace{-30pt}
\begin{IEEEbiography}[{\includegraphics[width=1in,height=1.25in,clip,keepaspectratio]{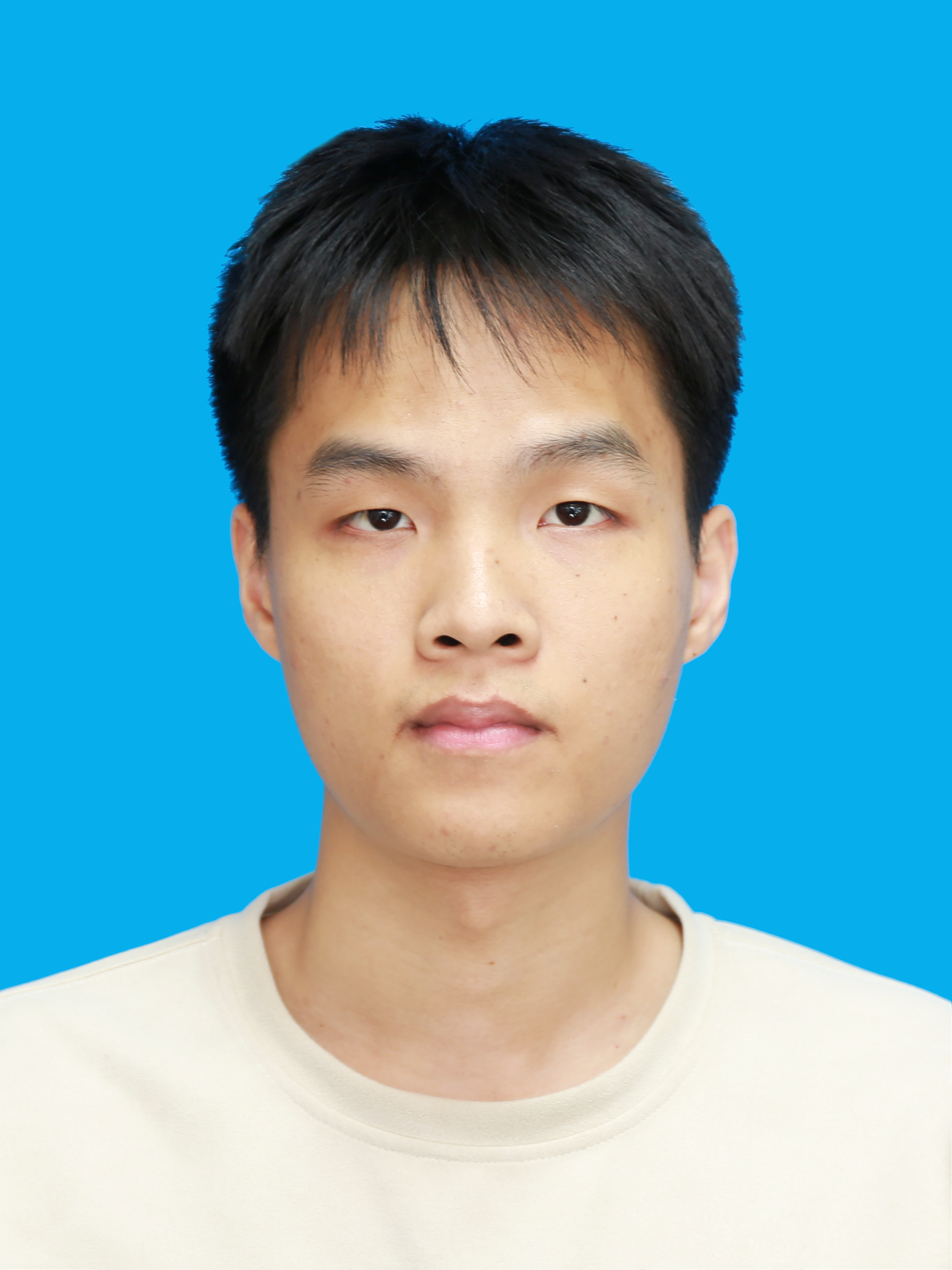}}]{Yibin~Yang}
   received the B.S. degree in  mechanical engineering from Tsinghua University, Beijing, China, 2019. He is currently working toward the Ph.D. degree in mechanical engineering with School of Vehicle and Mobility, Tsinghua, Beijing, China. His research interests include the multi-agent path finding, multi-vehicle trajectory planning, trajectory prediction and planning for autonomous driving.
\end{IEEEbiography}

\vspace{-30pt}
\begin{IEEEbiography}[{\includegraphics[width=1in,height=1.25in,clip,keepaspectratio]{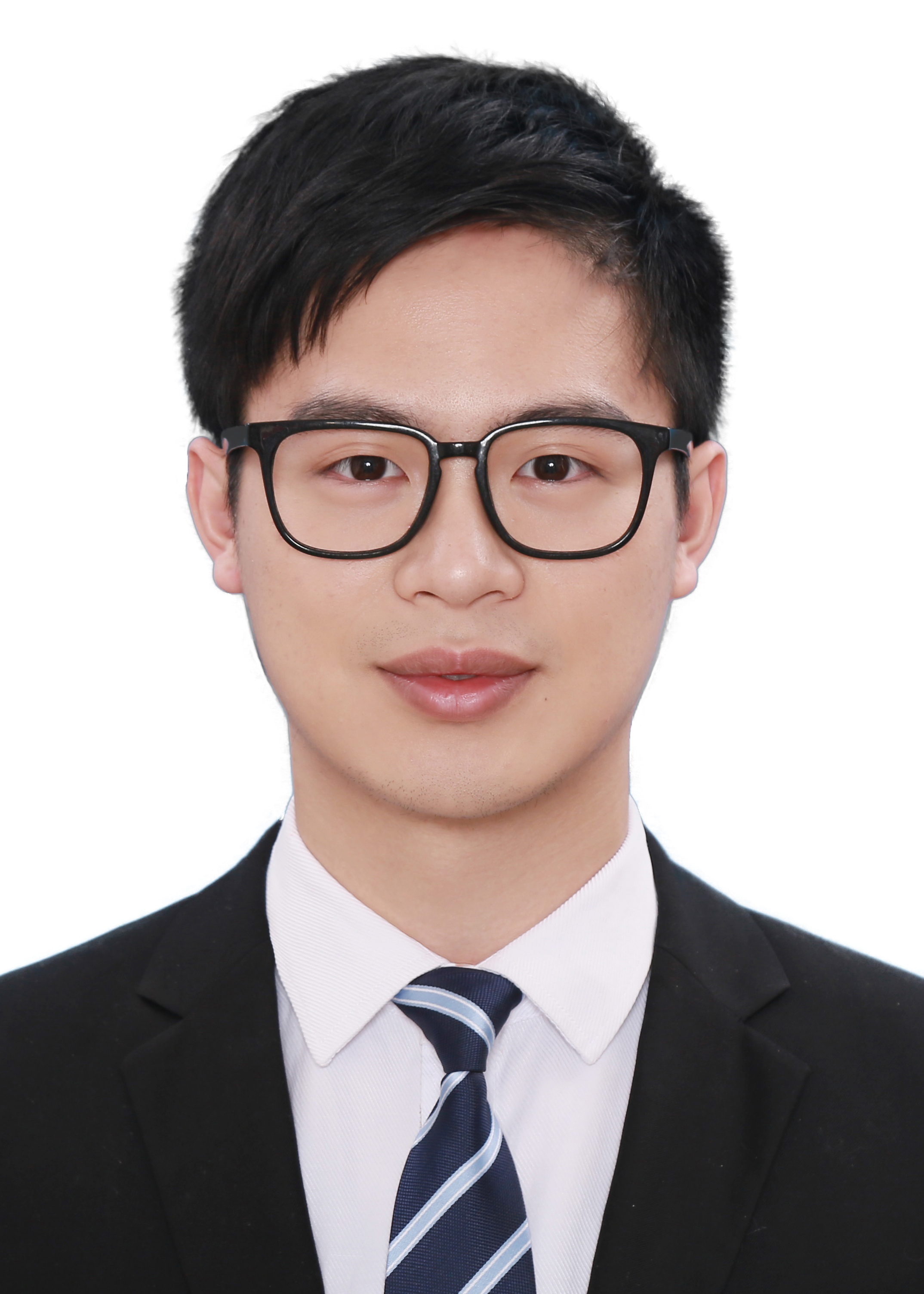}}]{Yihe~Chen}
   received the B.E. degree from Tsinghua University, Beijing, China, in 2021. He is currently pursuing the Ph.D. degree in mechanical engineering with the School of Vehicle and Mobility, Tsinghua University. His research interests include autonomous intersection management, multi-intersections coordination, traffic signal phasing and timing optimization, and motion planning of intelligent vehicles.
\end{IEEEbiography}

\vspace{-30pt}
\begin{IEEEbiography}[{\includegraphics[width=1in,height=1.25in,clip,keepaspectratio]{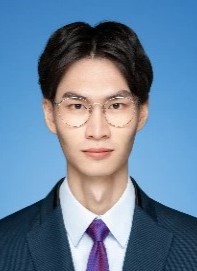}}]{Yuning~Wang}
   received the bachelor’s degree in automotive engineering from School of Vehicle and Mobility, Tsinghua University, Beijing, China, in 2020. He is currently pursuing the Ph.D. degree in mechanical engineering with School of Vehicle and Mobility, Tsinghua University, Beijing, China. His research centered on scene understanding, decision-making and planning, and driving evaluation of intelligent vehicles.
\end{IEEEbiography}

\vspace{-30pt}
\begin{IEEEbiography}[{\includegraphics[width=1in,height=1.25in,clip,keepaspectratio]{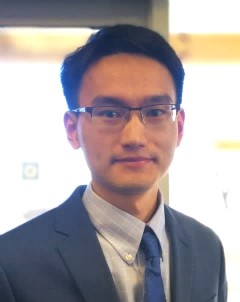}}]{Shaobing~Xu}
   received his Ph.D. degree in Mechanical Engineering from Tsinghua University, Beijing, China, in 2016. He is currently an assistant professor with the School of Vehicle and Mobility at Tsinghua University, Beijing, China. He was an assistant research scientist and postdoctoral researcher with the Department of Mechanical Engineering and Mcity at the University of Michigan, Ann Arbor. His research focuses on vehicle motion control, decision making, and path planning for autonomous vehicles. He was a recipient of the outstanding Ph.D. dissertation award of Tsinghua University and the Best Paper Award of AVEC’2018.
\end{IEEEbiography}

\vspace{-30pt}
\begin{IEEEbiography}[{\includegraphics[width=1in,height=1.25in,clip,keepaspectratio]{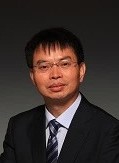}}]{Jianqiang~Wang}
   received the B.Tech. and M.S. degrees from Jilin University of Technology, Changchun, China, in 1994 and 1997, respectively, and the Ph.D. degree from Jilin University, Changchun, in 2002. He is currently a Professor and the Dean of the School of Vehicle and Mobility, Tsinghua University, Beijing, China. 

He has authored over 150 papers and is a co-inventor of over 140 patent applications. He was involved in over 10 sponsored projects. His active research interests include intelligent vehicles, driving assistance systems, and driver behavior. He was a recipient of the Best Paper Award in the 2014 IEEE Intelligent Vehicle Symposium, the Best Paper Award in the 14th ITS Asia Pacific Forum, the Best Paper Award in the 2017 IEEE Intelligent Vehicle Symposium, the Changjiang Scholar Program Professor in 2017, the Distinguished Young Scientists of NSF China in 2016, and the New Century Excellent Talents in 2008.
\end{IEEEbiography}

\end{document}